%% file: main.tex
\documentclass{ecai2014}
\usepackage{times}
\usepackage{graphicx}
\usepackage{latexsym}
\usepackage{helvet}
\usepackage{courier}

\input{src/meta/packages.tex}

\input{src/meta/theorems.tex}


\input{src/meta/graphik.tex}

\input{src/meta/dependencies-graph.tex}

\title{Extending Acyclicity Notions for Existential Rules \\(\emph{long version})}

\author{ Jean-Fran\c{c}ois Baget\institute{INRIA, France}
\and Fabien Garreau\institute{University of Angers, France}
\and Marie-Laure Mugnier\institute{University of Montpellier, France}
\and Swan Rocher$^3$
}

\begin{document}
\input{src/content-long.tex}

\end{document}

%% file: src/meta/packages.tex
\usepackage[utf8]{inputenc}
\usepackage{listings}
\usepackage{lscape}
\usepackage{url}
\usepackage{enumerate}
\usepackage{amssymb}
\usepackage{tikz}
\usepackage{xcolor}
\usepackage[ruled,vlined]{algorithm2e}
\usepackage{array}

\usetikzlibrary{shapes}
\usetikzlibrary{arrows}


%% file: src/meta/theorems.tex
\newtheorem{lemma}{Lemma}
\newtheorem{proposition}{Proposition}
\newtheorem{definition}{Definition}
\newtheorem{example}{Example}

\newcommand{\fun}[1]{\ensuremath\mbox{\sl #1}}

\newenvironment{proof}{
	\noindent {\em Proof:}
}{
	\hfill$\square$
}

\newenvironment{observation}{
	\noindent {\em Observation:}
}{ }

\newcounter{problem}

\newcolumntype{L}[1]{>{\raggedright\let\newline\\\arraybackslash\hspace{0pt}}m{#1}}
\newcolumntype{C}[1]{>{\centering\let\newline\\\arraybackslash\hspace{0pt}}m{#1}}
\newcolumntype{R}[1]{>{\raggedleft\let\newline\\\arraybackslash\hspace{0pt}}m{#1}}

%% file: src/meta/graphik.tex
\newcommand{\NP}[0]{NP}
\newcommand{\NPc}[0]{\NP-complete}
\newcommand{\coNP}[0]{co-\NP}
\newcommand{\coNPc}[0]{co-\NPc}

\newcommand{\R}[0]{\mathcal R}

\newcommand{\unifier}[0]{\mu}
\newcommand{\unified}[0]{\diamond}

\newcommand{\gerd}[1]{#1^{U}}
\newcommand{\grd}[1]{#1^{D}}
\newcommand{\gerdc}[1]{#1^{U+}}

\newcommand{\wa}[0]{wa}

\newcommand{\fd}[0]{fd}

\newcommand{\ar}[0]{ar}

\newcommand{\agrd}[0]{aGRD}

\newcommand{\ja}[0]{ja}
\newcommand{\swa}[0]{swa}
\newcommand{\msa}[0]{msa}
\newcommand{\mfa}[0]{mfa}

\newcommand{\pg}[0]{PG}
\newcommand{\ppg}[1]{\pg^{#1}}
\newcommand{\xpg}[0]{\ppg{X}}
\newcommand{\fpg}[0]{\ppg{F}}
\newcommand{\dpg}[0]{\ppg{D}}
\newcommand{\upg}[0]{\ppg{U}}
\newcommand{\PG}[1]{\pg(#1)}

\newcommand{\XPG}[1]{\xpg(#1)}
\newcommand{\FPG}[1]{\fpg(#1)}
\newcommand{\DPG}[1]{\dpg(#1)}
\newcommand{\UPG}[1]{\upg(#1)}

\newcommand{\vxpg}[4]{ $[$#4,#3$]$ }
\newcommand{\vpos}[2]{\ensuremath{\vxpg{}{}{#2}{#1}}}

\newcommand{\Mark}[0]{\ensuremath{M}}	

\renewcommand{\t}[0]{\vec t}

\newcommand{\X}[0]{\vec x}
\newcommand{\Y}[0]{\vec y}
\newcommand{\Z}[0]{\vec z}



%% file: src/meta/dependencies-graph.tex
\newcounter{gerdRuleID}
\newcounter{gerdTermID}
\newcounter{gerdUnifID}
\newcounter{gerdNbBodyTerms}
\newcounter{gerdNbHeadTerms}

\newenvironment{gerd:figure}{
	\setcounter{gerdRuleID}{0}
	\setcounter{gerdUnifID}{0}
	\begin{tikzpicture}
}{
	\end{tikzpicture}
}

\tikzstyle{gerd:node}=[
]

\tikzstyle{gerd:fake}=[
	gerd:node,
	inner sep=0mm
]

\tikzstyle{gerd:term}=[
	gerd:node,
	circle,
	minimum size=1mm,
	inner sep=0mm,
	draw=none
]

\tikzstyle{gerd:edge}=[
	->,
	draw=black
]

\tikzstyle{gerd:edge-norm}=[
	gerd:edge,
]

\tikzstyle{gerd:edge-spec}=[
	gerd:edge-norm
]

\tikzstyle{gerd:edge-unif}=[
	gerd:edge,
]

\tikzstyle{gerd:border}=[
	-,
	draw=black
]

\tikzstyle{gerd:border-ext}=[
	gerd:border,
	thick
]

\tikzstyle{gerd:border-int}=[
	gerd:border,
	draw=black!50!white,
	dashed
]

\newcounter{gerdTmpCounter}
\newcounter{gerdMaxTermCounter}
\newcommand{\gerdRule}[9][0]{
	\refstepcounter{gerdRuleID}

	\setcounter{gerdNbBodyTerms}{0}
	\setcounter{gerdNbHeadTerms}{0}
	\foreach \tb in {#6} { \refstepcounter{gerdNbBodyTerms} }
	\foreach \th in {#7} { \refstepcounter{gerdNbHeadTerms} }

	\foreach \i [evaluate=\i as \x using (((#4-#2)/4)*\i)+#2] in {0, 1, 2, 3, 4} {
		\node[gerd:fake] (fakeT\thegerdRuleID_\i) at (\x,#3) {};
	}
	\foreach \i [evaluate=\i as \x using (((#4-#2)/4)*\i)+#2, evaluate=\i as \y using #3+#3-#5] in {0, 1, 2, 3, 4} {
		\node[gerd:fake] (fakeB\thegerdRuleID_\i) at (\x,\y) {};
	}

	\draw[gerd:border-ext] (fakeT\thegerdRuleID_0.center) -- (fakeT\thegerdRuleID_4.center)
						-- (fakeB\thegerdRuleID_4.center) -- (fakeB\thegerdRuleID_0.center)
						-- (fakeT\thegerdRuleID_0.center);
	\draw[gerd:border-int] (fakeT\thegerdRuleID_2.center) -- (fakeB\thegerdRuleID_2.center);

	\ifthenelse{\equal{#1}{1}}{
		\pgfmathparse{#2+(3*((#4-#2)/4))}
	}{
		\pgfmathparse{#2+((#4-#2)/4)}
	}
	\setcounter{gerdTermID}{1}
	\foreach \pp/\pt/\ppl [evaluate=\pp as \x using \pgfmathresult,
				  	  evaluate=\pp as \y using #3-(((#5-#3)/(\thegerdNbBodyTerms+1))*\thegerdTermID)]
			in {#6} {	
				\setcounter{gerdTmpCounter}{1}
				\setcounter{gerdMaxTermCounter}{0}
				\foreach \ptt in \pt { \refstepcounter{gerdMaxTermCounter} }
		\node[gerd:term]
			(tb\thegerdRuleID_\pp{\pt}\ppl) at (\x,\y) {$\pp(
				\foreach \ptt in \pt {
					\ifthenelse{\equal{\thegerdTmpCounter}{\ppl}}{{\underline{\bf \ptt}}}{\ptt}
					\ifthenelse{\equal{\thegerdTmpCounter}{\thegerdMaxTermCounter}}{}{,}
					\refstepcounter{gerdTmpCounter}
				}
				)$};
		\refstepcounter{gerdTermID}
	}

	\ifthenelse{\equal{#1}{1}}{
		\pgfmathparse{#2+((#4-#2)/4)}
	}{
		\pgfmathparse{#2+(3*((#4-#2)/4))}
	}
	\setcounter{gerdTermID}{1}
	\foreach \pp/\pt/\ppl [evaluate=\pp as \x using \pgfmathresult,
				  evaluate=\pp as \y using #3-(((#5-#3)/(\thegerdNbHeadTerms+1))*\thegerdTermID)]
			in {#7} {	
				\setcounter{gerdTmpCounter}{1}
				\setcounter{gerdMaxTermCounter}{0}
				\foreach \ptt in \pt { \refstepcounter{gerdMaxTermCounter} }
		\node[gerd:term]
			(th\thegerdRuleID_\pp{\pt}\ppl) at (\x,\y) {$\pp(
				\foreach \ptt in \pt {
					\ifthenelse{\equal{\thegerdTmpCounter}{\ppl}}{{\underline{\bf \ptt}}}{\ptt}
					\ifthenelse{\equal{\thegerdTmpCounter}{\thegerdMaxTermCounter}}{}{,}
					\refstepcounter{gerdTmpCounter}
				}
				)$};
		\refstepcounter{gerdTermID}
	}

	\foreach \ppa/\ppb in {#8} {
		\draw[gerd:edge-norm] (tb\thegerdRuleID_\ppa.east) -- (th\thegerdRuleID_\ppb.west);
		\foreach \te in {#9} {
			\draw[gerd:edge-spec] (tb\thegerdRuleID_\ppa.east) -- (th\thegerdRuleID_\te.west);
		}
	}
}

\newcommand{\gerdUnif}[5][]{
	\foreach \th/\tb in {#4} {
		\draw[gerd:edge-unif,#1] (th#2_\th) to node[above]{#5} (tb#3_\tb);
	}
}

\newenvironment{gp:figure}{
	\begin{tikzpicture}
}{
	\end{tikzpicture}
}

\tikzstyle{gp:node}=[
	gerd:node
]

\tikzstyle{gp:fake}=[
	gp:node,
	gerd:fake
]

\tikzstyle{gp:predpos}=[
	gp:node,
	rectangle,
	minimum size=1mm,
	draw=black
]

\tikzstyle{gp:edge}=[
	gerd:edge
]

\tikzstyle{gp:edge-norm}=[
	gp:edge,
	gerd:edge-norm
]

\tikzstyle{gp:edge-spec}=[
	gp:edge,
	gerd:edge-spec
]

\tikzstyle{fake}=[
	inner sep=0mm
]

\tikzstyle{rc:node}=[
]

\tikzstyle{rc:edge}=[
	-,
	draw=black
]

\tikzstyle{rc:complexity-edge}=[
	-,
	thick,
	densely dotted,
	draw
]

\tikzstyle{rc:P-edge}=[
	rc:complexity-edge,
	draw=green!66!black
]

\tikzstyle{rc:NP-edge}=[
	rc:complexity-edge,
	draw=blue!66!black
]

\tikzstyle{rc:ET-edge}=[
	rc:complexity-edge,
	draw=orange!66!black
]

\tikzstyle{rc:ETT-edge}=[
	rc:complexity-edge,
	draw=red!66!black
]

\tikzstyle{rc:complexity-label}=[
]

%% file: src/content-long.tex

\maketitle
\noindent \emph{This report contains a revised version (July 2014) of the paper that will appear in the proceedings of ECAI 2014 and an appendix with proofs that could not be included in the paper for space restriction reasons. 
}
\vspace*{0.5 cm}

\begin{abstract}
\input{src/content/abstract.tex}
\end{abstract}

\section{INTRODUCTION}
\input{src/content/intro.tex}


\section{PRELIMINARIES}
\label{sec:preliminaries}

\input{src/content/preliminaries.tex}


\section{KNOWN ACYCLICITY NOTIONS}
\label{sec:review}
\input{src/content/review.tex}

\section{EXTENDING ACYCLICITY NOTIONS}
\label{sec:acyclic}
\input{src/content/acyclicity.tex}

\section{FURTHER REFINEMENTS}
\label{sec:unif}
\input{src/content/unif.tex}

\section{EXTENSION TO NONMONOTONIC NEGATION}

\label{sec:negation}
\input{src/content/negation.tex}

\section{CONCLUSION}
\input{src/content/conclusion.tex}

\paragraph{Acknowledgements. }
This work was partially supported by French \emph{Agence Nationale de la Recherche} (ANR), under project grants ASPIQ (ANR-12-BS02-0003), Pagoda (ANR-12-JS02-0007) and Qualinca  (ANR-12-CORD-0012). We thank Michael Thomazo for pointing out a flaw in the original definitions of $\DPG{\R}$ and $\UPG{\R}$.

\bibliography{src/bib}
\bibliographystyle{ecai2014}

\pagebreak
\section*{Appendix}
\input{src/content/appendix.tex}

%% file: src/content/abstract.tex
Existential rules have been proposed for representing ontological
knowledge, specifically in the context of Ontology-Based Query Answering. 
Entailment with existential rules is undecidable.
We focus in this paper on conditions that ensure
the termination of a breadth-first forward chaining algorithm
known as the chase. First, we propose a new tool that allows to extend 
existing acyclicity conditions ensuring chase termination, while keeping
good complexity properties. Second, we consider the extension
to existential rules with nonmonotonic negation under
stable model semantics and further extend acyclicity results
obtained in the positive case.

%% file: src/content/intro.tex
Ontology-Based Query Answering is a new paradigm in data management, which aims to exploit ontological knowledge when accessing
data. \emph{Existential rules} have been proposed for representing ontological knowledge, specifically in this context \cite{cali09,blms09}.  
These rules allow to assert the existence of unknown individuals,
an essential feature in an open-domain perspective. They
generalize lightweight description logics, such as DL-Lite and $\mathcal{EL}$ \cite{dl-lite07,baader-b-l05} and overcome some of their
limitations by allowing any predicate arity as well as cyclic structures.

In this paper, we focus on a breadth-first forward chaining algorithm, known as the
\emph{chase} in the database literature \cite{maier79}. The chase was originally used in the context of very general database constraints, called tuple-generating dependencies, which have the same logical form as existential rules \cite{beeri-vardi84}. 

Given a knowledge base composed of data
and existential rules, the chase triggers the rules and materializes performed inferences  in
the data. The ``saturated'' data can then be queried like a classical database. This allows to
benefit from optimizations implemented in current data management systems.   However, the chase
is not ensured to terminate--- which applies to any sound and complete mechanism, since
entailment with existential rules is undecidable (\cite{beeri-vardi81,chandra-lewis-makowsky81}  on tuple-generating dependencies). 
Various acyclicity notions ensuring chase termination have been proposed in knowledge representation and database theory.


\paragraph{Paper contributions.} 

We generalize known acyclicity conditions, first, for plain existential rules, second, for their extension to nonmonotonic negation with stable semantics. 

\emph{1. Plain existential rules.} Acyclicity conditions found in the literature can be classified into two main families: the first one constrains the way existential variables are propagated during the chase, e.g.,  \cite{fagin-kolaitis-al05,marnette09,kr11}, and the second one constrains dependencies between rules, i.e., the fact that a rule may lead to trigger another rule, e.g., \cite{baget04,deutsch08,blms11,grau2013acyclicity}. These conditions are based on different graphs, but all of them can be seen as forbidding ``dangerous'' cycles in the considered graph. We define  a new family of graphs that allows to unify and strictly generalize these acyclicity notions, without increasing worst-case complexity. 

\emph{2. Extension to nonmonotonic negation.} Nonmonotonic negation is a useful feature in ontology modeling. Nonmontonic extensions to existential rules were recently considered in \cite{cali09} with stratified negation,
\cite{pods2013-hernich} with well-founded semantics and \cite{mkh13} with stable model semantics. The latter paper
focuses on cases where a unique finite model exists; we consider the same rule framework, however without enforcing a unique model. We further extend acyclicity results obtained on positive rules by exploiting negative information as well. 

The paper is organized according to these two issues. 

%% file: src/content/preliminaries.tex
An \emph{atom} is of the form $p(t_1, \ldots, t_k)$ where $p$ is a predicate of arity $k$ and
the $t_i$ are terms, i.e., variables or constants. An \emph{atomset}  is a finite set of atoms.
If $F$ is an atom or an atomset, we denote by $\fun{terms}(F)$ (resp. $\fun{vars}(F)$) the set of terms (resp. variables) that occur in $F$. 
 In the examples illustrating the paper, all the terms are variables (denoted by $x$, $y$, $z$, etc.), unless otherwise specified. 
  Given atomsets  $A_1$ and $A_2$, a \emph{homomorphism}  $h$ from $A_1$ to $A_2$ 
is a substitution of $\fun{vars}(A_1)$ by $\fun{terms}(A_2)$ such that $h(A_1) \subseteq A_2$.

An \emph{existential rule} (and simply a rule hereafter) is of the form $R = \forall \vec{x}\forall \vec{y} (B \rightarrow \exists \vec{z} H)$, where $B$ and $H$ are conjunctions of atoms, with $\fun{vars}(B) = \vec{x} \cup \vec{y}$, and $\fun{vars}(H) = \vec{x} \cup \vec{z}$. $B$ and $H$ are respectively called the \emph{body} and  the  \emph{head} of $R$. We also use the notations  $\fun{body}(R)$ for $B$ and $\fun{head}(R)$ for $H$. 
Variables $\vec{x}$, which appear in both $B$ and $H$, are called \emph{frontier variables}. Variables $\vec{z}$, which appear only in $H$, are called \emph{existential variables}. Hereafter, we  omit quantifiers in rules as there is no ambiguity. E.g., $p(x,y) \rightarrow p(y,z)$
stands for $\forall x \forall y (p(x,y) \rightarrow \exists z (p(y,z)))$. 
 
 An existential rule with an empty body is called a \emph{fact}. A fact is thus an existentially closed conjunction of atoms. A \emph{Boolean conjunctive query} (BCQ) has the same form.  A \emph{knowledge base} $\mathcal K = (F, \mathcal R)$ is composed of a finite set of
facts (which is seen as a single fact) $F$ and a finite set of existential rules $\mathcal R$. The fundamental problem associated with query answering, called {\sc BCQ entailment}, is the following: given a knowledge base $(F, \mathcal R)$ and a BCQ $Q$, is it true that $F, \mathcal R \models Q$, where $\models$ denotes the standard logical consequence?  This problem is undecidable (which follows from the undecidability of the implication problem on tuple-generating dependencies \cite{beeri-vardi81,chandra-lewis-makowsky81}). 

 In the following, we see conjunctions of atoms as atomsets.  A rule $R : B \rightarrow H$ is \emph{applicable} to an atomset $F$ if there is a homomorphism $\pi$ from $B$ to $F$.  The \emph{application of $R$ to $F$ w.r.t. $\pi$} produces an atomset $\alpha(F, R, \pi) = F \cup \pi(\fun{safe}(H))$, where $\fun{safe}(H)$ is obtained from $H$ by replacing existential variables with fresh variables (see Example \ref{ex-chase}).

The \emph{chase} procedure starts from the initial set of facts $F$ and performs rule applications in a breadth-first
manner.   Several chase variants can be found in the literature, mainly \emph{oblivious} (or naive),  e.g., \cite{cali-gottlob-kifer08}, \emph{skolem} \cite{marnette09}, \emph{restricted} (or standard) \cite{fagin-kolaitis-al05}, and \emph{core} chase \cite{deutsch08}. The oblivious chase performs all possible rule applications.  The \emph{skolem chase} relies on a skolemisation of the rules (notation $\fun{sk}$): for each rule $R$, $\fun{sk}$($R$) is obtained from $R$ by replacing each occurrence of an existential variable $y$ with a functional term $f^R_y(\vec{x})$, where $\vec x$ is the set of frontier variables in $R$. Then, the oblivious chase is run on skolemized rules. 

\begin{example} [Oblivious / Skolem chase] \label{ex-chase}
Let $R = p(x,y) \rightarrow p(x,z)$ and $F = \{p(a,b)\}$, where $a$ and $b$ are constants.
The oblivious chase does not halt: it applies $R$ according to $h_0 = \{(x,a), (y,b)\}$, hence adds $p(a,z_0)$; then, it applies $R$ again according to $h_1 = \{(x,a), (y,z_0)\}$, and adds $p(a,z_1)$, and so on. The skolem chase considers the rule $p(x,y) \rightarrow p(x,f^R_z(x))$; it adds $p(a, f^R_y(a))$ then halts.
\end{example}

Due to space restrictions, we do not detail on the other chase variants. Given a chase variant $C$, we call \emph{$C$-finite} the class of  set of rules $\mathcal R$, such that the $C$-chase halts on $(F, \mathcal R)$ for any atomset $F$. It is well-known that oblivious-finite $\subset$ skolem-finite $\subset$ restricted-finite $\subset$  core-finite (see, e.g.,  \cite{onet-2013}). When $\mathcal R$ belongs to a $C$-finite class, {\sc BCQ entailment} can be solved, for any $F$ and $Q$, by running the $C$-chase on $(F, \mathcal R)$, which produces a saturated set of facts $F^*$, then checking if $F^* \models Q$.

%% file: src/content/review.tex

Acyclicity notions can be divided into two main families, each of them relying on a different graph. The first family relies on a graph encoding variable sharing between \emph{positions} in predicates, while the second one relies on a graph encoding \emph{dependencies} between rules, i.e., the fact that a rule may lead to trigger another rule (or itself). 

\subsection{Position-based approach}

In the position-based approach, dangerous cycles are those passing through positions that may contain
existential variables; intuitively, such a cycle means that the creation of an existential variable in a given
position may lead to creating another existential variable in the same position, hence an infinite number of existential
variables. Acyclicity is then defined by the absence of dangerous cycles.  The simplest acyclicity notion in this
family is that of  \emph{weak-acyclicity} \emph{(wa)} \cite{fagin-kolaitis-al05}, which has been widely used
in databases. It relies on a directed graph whose nodes are the positions in predicates (we denote by $(p,i)$ position $i$ in predicate $p$).
 Then, for each rule $R: B \rightarrow H$ and each frontier variable $x$ in $B$ occurring in position $(p,i)$, edges with origin $(p,i)$ are built as follows:
 there is an edge from $(p,i)$ to each position of $x$ in $H$;
 furthermore, for each existential variable $y$ in $H$ occurring in position $(q,j)$, there is a special edge from $(p,i)$ to $(q,j)$.
 A set of rules is weakly-acyclic if its associated graph has no cycle passing through a special edge (see Example \ref{ex-wa}). 
 
 \begin{example}[Weak-acyclicity]\label{ex-wa}
Let $R_1 = h(x) \rightarrow p(x,y)$ and
$R_2 =  p(u,v), q(v) \rightarrow h(v)$. The position graph of $\{R_1, R_2\}$ contains a special edge from $(h,1)$ to $(p,2)$ due to $R_1$ and an edge from $(p,2)$ to $(h,1)$ due to $R_2$. Thus $\{R_1, R_2\}$ is not wa. 
\end{example}
 
 Weak-acyclicity  has been generalized, mainly by shifting the focus from positions to existential variables
(\emph{joint-acyclicity} \emph{(ja)}\cite{kr11}) or to positions in atoms instead of predicates
(\emph{super-weak-acyclicity} \emph{(swa)} \cite{marnette09}). Other related notions can be imported from logic programming,
e.g., \emph{finite domain} \emph{(fd)} \cite{fd2008}  and \emph{ argument-restricted} \emph{(ar)} \cite{ar2009}.
See the first column in
Figure~\ref{fig:gen}, which shows the inclusions between the corresponding classes of rules; all these inclusions are known to be strict. 

\subsection{Rule dependency-based approach}

In the second approach, the aim is to avoid cyclic triggering of rules \cite{baget04,deutsch08,blms09,blms11,grau2013acyclicity}.
 We say that a rule $R_2$ \emph{depends} on a rule $R_1$ if an application of $R_1$ may lead to a new application of $R_2$: there exists an atomset $F$ such that $R_1$ is applicable to $F$ with homomorphism $\pi$ and $R_2$ is applicable to $F' = \alpha(F, R_1, \pi)$ with homomorphism $\pi'$, which is \emph{new} ($\pi'$ is not a homomorphism to $F$) and \emph{useful} ($\pi'$ cannot be extended to a homomorphism from $H_2$ to $F'$).
This abstract dependency relation can be computed with a unification operation known as piece-unifier
\cite{blms09}.  Piece-unification  takes existential variables into account, hence is more complex than the usual
unification between atoms. A \emph{piece-unifier} 
of a rule body $B_2$ with a
rule head $H_1$ is a substitution $\mu$ of $\fun{vars}(B'_2) \cup \fun{vars}(H'_1)$, where $B'_2 \subseteq B_2$ and $H'_1
\subseteq H_1$, such that: \emph{(1) }$\mu(B'_2) = \mu(H'_1)$, and \emph{(2)} existential variables in $H'_1$ are not
unified with \emph{separating} variables of $B'_2$, i.e., variables that occur both in $B'_2$ and in $(B_2 \setminus B'_2)$; in
other words, if a variable $x$ occuring in $B'_2$ is unified with an existential variable $y$ in $H'_1$, then all atoms
in which $x$ occur also belong to $B'_2$. It holds that $R_2$ depends on $R_1$ iff there is a piece-unifier of $B_2$ with
$H_1$, satisfying some easily checked additional conditions (atom erasing \cite{blms11} and usefulness
\cite{lamare12,grau2013acyclicity}).
Following Example \ref{ex-dependency} illustrates the difference between piece-unification and classical unification.

\begin{example}[Rule dependency]\label{ex-dependency}
Let $R_1$ and $R_2$ from Example  \ref{ex-wa}. Although the atoms  $p(u,v) \in B_2$ and $p(x,y) \in H_1$ are unifiable, there is no piece-unifier of $B_2$ with $H_1$. Indeed, the most general unifier $\mu = \{(u,x), (v,y)\}$ (or, equivalently, $\{(x,u), (y,v)\}$), with $B'_2 = \{p(u,v)\}$  and $H'_1 = H_1$, is not a piece-unifier because $v$ is unified with an existential variable, whereas it is a separating variable of $B'_2$ (thus, $q(v)$ should be included in $B'_2$). It follows that $R_2$ does not depend on $R_1$. 
\end{example}

The \emph{graph of rule dependencies } of a set of rules $\mathcal R$, denoted by GRD($\mathcal R)$, is a directed graph with set of nodes $\mathcal R$ and an edge $(R_i, R_j)$ if $R_j$ depends on $R_i$. E.g., with the rules in Example \ref{ex-dependency}, the only edge is $(R_2,R_1)$. When the GRD is acyclic (\emph{aGRD}, \cite{baget04}), any derivation sequence is  finite.

\subsection{Combining both approches}

 Both approaches are incomparable: there may be a dangerous cycle on positions but no cycle w.r.t. rule dependencies (Example \ref{ex-wa} and \ref{ex-dependency}), and there may be a cycle w.r.t. rule dependencies whereas rules have no existential variables (e.g., $p(x,y) \rightarrow p(y,x)$).   
So far, attempts to combine both notions only succeded to combine them in a ``modular way'', by considering the strongly connected components (s.c.c.) of the GRD  \cite{baget04,deutsch08}; briefly, if a chase variant stops on each subset of rules associated with a s.c.c., then it stops on the whole set of rules. 
%
In this paper, we propose an ``integrated'' way of combining both approaches, which relies on a single graph. This
allows to unify preceding results and to generalize them without increasing complexity. The new acyclicity notions are
those with a gray background in Figure~\ref{fig:gen}.

Finally, let us mention \emph{model-faithful acyclicity}\emph{ (mfa)} \cite{grau2013acyclicity}, which
 cannot be captured by our approach. Briefly, checking \emph{mfa} involves running the skolem chase
until termination or a cyclic functional term is found. The price to pay  is high complexity:
checking if a set of rules is model-faithful acyclic for any set of facts is 2EXPTIME-complete. Checking
\emph{model-summarizing acyclicity (msa)} \cite{grau2013acyclicity}, which approximates mfa, remains EXPTIME-complete.
In contrast, checking position-based
properties is in PTIME and checking aGRD is co-NP-complete.    

\begin{center}
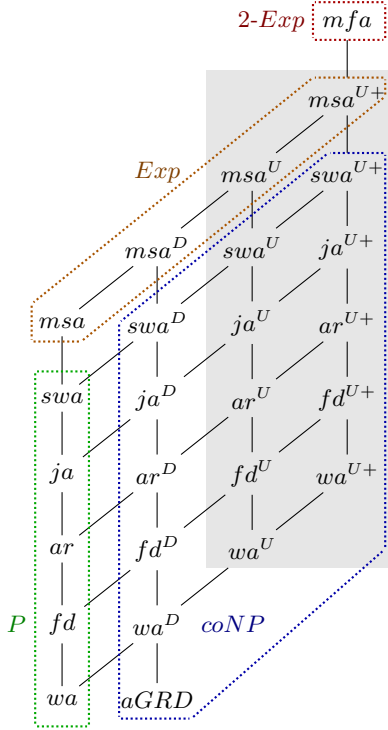
\begin{figure}[ht!]
	\input{src/figures/classes-relations.tex}
	\caption{Relations between recognizable acyclicity properties. All inclusions are strict and complete (i.e., if there is no path between two properties then they are incomparable). }
\label{fig:gen}
\end{figure}
\end{center}

It remains to specify for which chase variants the above acyclicity notions ensure termination. Since \emph{mfa} generalizes all properties in Figure~\ref{fig:gen}, and sets of rules satisfying \emph{mfa} are skolem-finite, all these properties ensure $C$-finiteness, for any chase variant $C$ at least as strong as the skolem chase.
We point out that the oblivious chase may not stop on \emph{wa} rules.
 Actually, the only acyclicity notion in Figure~\ref{fig:gen} that ensures the termination of the
 oblivious chase is  \emph{aGRD}, since all other notions generalize \emph{wa}.

%% file: src/figures/classes-relations.tex

\begin{tikzpicture}[node distance=1cm]
	\node[fake] (contrib0) at (1.9,1.75) {};
	\node[fake] (contrib1) at (4.35,1.75) {};
	\node[fake] (contrib2) at (4.35,8.35) {};
	\node[fake] (contrib3) at (1.9,8.35) {};
	\draw[draw=none,fill=black!10] (contrib0.center) -- (contrib1.center) -- (contrib2.center) -- (contrib3.center) --
	(contrib0.center);

	\node[rc:node] (wa) at (0,0) {$\wa$};

	\node[rc:node,node distance=1.25cm] (agrd) [right of=wa] {$\agrd$};
	\node[rc:node] (wad) [above of=agrd] {$\grd{\wa}$};

	\node[fake,node distance=1.25cm] (f0) [right of=wad] {};
	\node[rc:node] (wau) [above of=f0] {$\gerd{\wa}$};

	\node[fake,node distance=1.25cm] (f1) [right of=wau] {};
	\node[rc:node] (wac) [above of=f1] {$\gerdc{\wa}$};

	\foreach \t/\g in {/,d/\grd,u/\gerd,c/\gerdc}{
		\node[rc:node] (fd\t) [above of=wa\t]   {$\g{\fd}$};
		\node[rc:node] (ar\t) [above of=fd\t]   {$\g{\ar}$};
		\node[rc:node] (ja\t) [above of=ar\t]   {$\g{\ja}$};
		\node[rc:node] (swa\t) [above of=ja\t]  {$\g{\swa}$};
		\node[rc:node] (msa\t) [above of=swa\t] {$\g{\msa}$};
	}

	\node[rc:node] (mfa) [above of=msac] {$\mfa$};

	\foreach \c in {wa,fd,ar,ja,swa,msa} {
		\draw[rc:edge] (\c) -- (\c d);
		\draw[rc:edge] (\c d) -- (\c u);
		\draw[rc:edge] (\c u) -- (\c c);
	}

	\foreach \t in {,d,u,c} {
		\draw[rc:edge] (wa\t) -- (fd\t);
		\draw[rc:edge] (fd\t) -- (ar\t);
		\draw[rc:edge] (ar\t) -- (ja\t);
		\draw[rc:edge] (ja\t) -- (swa\t);
		\draw[rc:edge] (swa\t) -- (msa\t);
	}

	\draw[rc:edge] (agrd) -- (wad);
	\draw[rc:edge] (msac) -- (mfa);

	\node[fake,node distance=0.5cm
	] (p0) [below left of=wa] {};
	\node[fake,node distance=0.5cm
	] (p1) [below right of=wa] {};
	\node[fake,node distance=0.5cm] (p2) [above right of=swa] {};
	\node[fake,node distance=0.5cm] (p3) [above left of=swa] {};
	\draw[rc:P-edge] (p0) -- (p1) -- (p2) -- (p3) -- (p0);

	\node[fake] (n0) at (0.75,-0.25) {};
	\node[fake] (n1) at (1.7,-0.25) {};
	\node[fake] (n2) at (4.25,2) {};
	\node[fake] (n3) at (4.25,7.25) {};
	\node[fake] (n4) at (3.5,7.25) {};
	\node[fake] (n5) at (0.75,5) {};
	\draw[rc:NP-edge] (n0) -- (n1) -- (n2) -- (n3) -- (n4) -- (n5) -- (n0);

	\node[fake] (e0) at (-0.4,4.75) {};
	\node[fake] (e1) at (0.25,4.75) {};
	\node[fake] (e2) at (4.25,8) {};
	\node[fake] (e3) at (4.25,8.25) {};
	\node[fake] (e4) at (3.3,8.25) {};
	\node[fake] (e5) at (-0.4,5.25) {};
	\draw[rc:ET-edge] (e0) -- (e1) -- (e2) -- (e3) -- (e4) -- (e5) -- (e0);

	\node[fake] (ee0) at(3.3,8.75)  {};
	\node[fake] (ee1) at(4.25,8.75) {};
	\node[fake] (ee2) at(4.25,9.25) {};
	\node[fake] (ee3) at(3.3,9.25)  {};
	\draw[rc:ETT-edge] (ee0) -- (ee1) -- (ee2) -- (ee3) -- (ee0);

	\node[rc:complexity-label, node distance=0.6cm] [left of=fd] {{\color{green!50!black}$P$}};
	\node[rc:complexity-label] [right of=wad] {{\color{blue!50!black}$coNP$}};
	\node[rc:complexity-label] [above of=msad] {{\color{orange!50!black}$Exp$}};
	\node[rc:complexity-label] [left of=mfa] {{\color{red!50!black}$2$-$Exp$}};


\end{tikzpicture}

%% file: src/content/acyclicity.tex

In this section, we combine rule dependency and propagation of existential variables into a single graph.  
W.l.o.g. we assume that distinct rules do not share any variable.
Given an atom $a = p(t_1, \dots, t_k)$, the $i^{th}$ position in $a$ 
is denoted by 
$\vpos{a}{i}$, 
with $\fun{pred}(\vpos{a}{i}) = p$ and $\fun{term}(\vpos{a}{i}) = t_i$.
If $A$ is an atomset such that $a \in A$, we say that $\vpos{a}{i}$ is in $A$.
 If $\fun{term}(\vpos{a}{i})$ is an existential (resp. frontier) variable, $\vpos{a}{i}$
is called an \emph{existential} (resp. \emph{frontier}) position.
In the following, we use ``position graph'' as a generic name
to denote a graph whose nodes are positions in \emph{atoms}.

We first define the notion of a basic position graph, which takes each rule in isolation.  Then, by adding edges to this graph, we define three position graphs with increasing expressivity, i.e., allowing to check termination for increasingly larger classes of rules.

\begin{definition}[(Basic) Position Graph ($\pg$)] The \emph{position graph} of a rule $R : B \rightarrow H$ is the directed graph $\PG{R}$ defined as follows:
	\begin{itemize}
    \item 
    there is a node for each $\vpos{a}{i}$ in $B$ or in $H$;
    \item 
    for all frontier positions $\vpos{b}{i} \in B$ and all $\vpos{h}{j} \in H$, there is an edge from $\vpos{b}{i}$
	to $\vpos{h}{j}$ if $\fun{term}(\vpos{b}{i}) = \fun{term}(\vpos{h}{j})$ or if $\vpos{h}{j}$ is existential.
\end{itemize}
Given a set of rules $\R$, the \emph{basic position graph} of $\R$, denoted by $\PG{\R}$,  is the
disjoint union of $\PG{R_i}$, for all $R_i \in \R$.
\end{definition}

An existential position $\vpos{a}{i}$ is said to be \emph{infinite} if there is an atomset $F$
such that running the chase on $F$ produces an unbounded number of instantiations of $\fun{term}(\vpos{a}{i})$.
To detect infinite positions, we encode how variables may be ``propagated" among rules by adding
 edges to $\PG{\R}$, called \emph{transition edges}, which go from positions in rule heads to positions  in rule bodies.
 The set of transition edges has to be \emph{correct}: if an existential position 
$\vpos{a}{i}$ is infinite, there must be a cycle going through $\vpos{a}{i}$ in the graph.

We now define three position graphs by adding transition edges to $\PG{\R}$, namely $\FPG{\R}$, $\DPG{\R}$ and $\UPG{\R}$. All three graphs have correct sets of edges. Intuitively, $\FPG{\R}$ corresponds to the case where all rules
are supposed to depend on all rules; 
its set of cycles is in bijection with the set of cycles in the predicate position graph defining weak-acyclicity.
$\DPG{\R}$ encodes actual paths of rule dependencies. Finally, $\UPG{\R}$ adds information about the piece-unifiers themselves. This provides an accurate encoding of variable propagation from an atom position to another.

\begin{definition}[$\xpg$]\label{def-xpg}
	Let $\R$ be a set of rules.
	The three following position graphs are obtained from $\PG{\R}$ by adding a (transition)
	edge from each $k^{th}$ position $\vpos{h}{k}$ in a rule head $H_i$ to each $k^{th}$ position  $\vpos{b}{k}$ in
	a rule body $B_{j}$, with the  \emph{same predicate},
	provided that some condition is satisfied :
	\begin{itemize}
		\item {\em full PG}, denoted by $\FPG{\R}$: no additional condition;
		\item {\em dependency PG}, denoted by  $\DPG{\R}$: if $R_{j}$ depends directly or indirectly on $R_i$, i.e., if there is a path from $R_i$ to $R_j$ in GRD($\mathcal R)$;
			\item {\em PG with unifiers}, denoted by $\UPG{\R}$: if there is a piece-unifier $\mu$ of $B_j$ with the head of an agglomerated rule $R^j_i$
		         such that $\mu(\fun{term}([b,k])) = \mu(\fun{term}([h,k]))$, where $R^j_i$ is formally defined below (Definition~\ref{def-agg}).		
	\end{itemize}
\end{definition}

An agglomerated rule associated with $(R_i, R_j)$ gathers information about selected piece-unifiers along (some) paths from $R_i$ to (some) predecessors of $R_j$.

\begin{definition}[Agglomerated Rule] \label{def-agg}
Given $R_i$ and $R_j$ rules from $\mathcal R$, an agglomerated rule associated with $(R_i, Rj)$ has the following form:
$$R^j_i =  B_i \cup_{t \in T \subseteq \fun{terms}(H_i)} \fun{fr}(t) \rightarrow H_i$$
where $\fun{fr}$ is a new unary predicate that does not appear in $\mathcal R$, 
and the atoms $\fun{fr}(t)$ are built as follows. Let $\mathcal P$ be a non-empty set of paths from $R_i$ to direct predecessors of $R_j$ in GRD($\mathcal R)$. 
Let $P= (R_1, \ldots, R_n)$ be a path in $\mathcal P$. One can associate a rule $R^P$ with $P$ by building a
sequence $R_1= R^p_1, \ldots, R^p_n = R^P$ such that $\forall 1 \leq l < n$,
there is a piece-unifier $\mu_l$ of $B_{l+1}$ with the head of $R^p_l$, where the body
of $R^p_{l+1}$ is $B^p_{l} \cup \{\fun{fr}(t) \, | \, t \, \mbox{is a term of } \, H^p_l\, \mbox{unified in } \mu_l\}$, and the head of $R^p_{l+1}$ is $H_1$. Note that for all $l$, $H^p_l = H_1$, however, for $l \neq 1$, $R^p_l$ may have less existential variables than $R_l$ due to the added atoms. The agglomerated rule $R^j_i$ built from $\{R^P | P \in \mathcal P\}$ is $R^j_i = \bigcup_{P \in \mathcal P} R^P$. 
\end{definition}

The following inclusions follow from the definitions:

\begin{proposition}[Inclusions between $\xpg$]
	Let $\R$ be a set of rules.
	$\UPG{\R} \subseteq \DPG{\R} \subseteq \FPG{\R}$.
	Furthermore, $\DPG{\R} = \FPG{\R}$ if the transitive closure of $GRD(\R)$ is a complete graph.
\end{proposition}

\begin{example}[$PG^F$ and $PG^D$]\label{ex:fpg-dpg}
	Let $\mathcal R = \{R_1, R_2\}$ from Example \ref{ex-wa}.
	Figure~\ref{fig:fpg-dpg} pictures $\FPG{\R}$ and $\DPG{\R}$.	
	The dashed edges  belong to $\FPG{\R}$ but not to  $\DPG{\R}$. Indeed, $R_2$ does not depend on $R_1$.
	$\FPG{\R}$ has a cycle
	while $\DPG{\R}$ has not.
\end{example}

\begin{example}[$PG^D$ and $PG^U$]\label{ex:dpg-upg}
	Let $\R = \{R_1, R_2\}$, with  $R_1 =  t(x,y) \rightarrow p(z,y), q(y)$ and
		$R_2 = p(u,v),q(u) \rightarrow t(v,w)$. In Figure~\ref{fig:dpg-upg}, the dashed edges belong to $\DPG{\R}$ but not to $\UPG{\R}$.
	Indeed,  the only piece-unifier of $B_2$ with $H_1$ unifies $u$ and $y$.
	Hence, the cycle in $\DPG{\R}$ disappears in $\UPG{\R}$.
\end{example}

\begin{center}
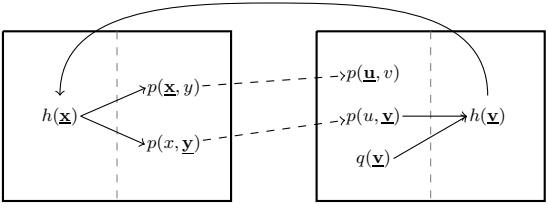
\begin{figure}[t]
\input{src/figures/expg2.tex}
\caption{$\FPG{\R}$ and $\DPG{\R}$ from Example \ref{ex:fpg-dpg}. Position $\vpos{a}{i}$ is
	represented by underlining the i-th term in $a$.
	Dashed edges do not belong to $\DPG{\R}$. 
}
\label{fig:fpg-dpg}
\end{figure}
\end{center}

\begin{center}
\begin{figure}[b]
\input{src/figures/expgS.tex}
\caption{$\DPG{\R}$ and $\UPG{\R}$ from Example \ref{ex:dpg-upg}.
	Dashed edges do not belong to $\UPG{\R}$. 
}
\label{fig:dpg-upg}
\end{figure}
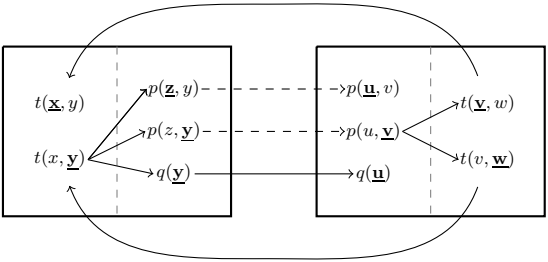
\end{center}



We now study how acyclicity properties can be expressed on position graphs. The idea is to associate, with an acyclicity property, a function that assigns to each position a subset of positions reachable from this position, according to some propagation constraints;  then, the property is fulfilled if  no existential position can be reached from itself.  
 More precisely, a \emph{marking function} $Y$ assigns to each node $\vpos{a}{i}$
	in a position graph $\xpg$, a subset of its (direct or indirect) successors,  
	called its {\em marking}. 
	A \emph{marked cycle} for $\vpos{a}{i}$ (w.r.t. $X$ and $Y$)
	is a cycle $C$ in $\xpg$ such that
	$\vpos{a}{i} \in C$ and for all $\vpos{a'}{i'} \in C$, $\vpos{a'}{i'}$ belongs to the marking of $\vpos{a}{i}$. 
  Obviously, the less situations there are in which the marking may ``propagate'' in a position graph,  
  the stronger the acyclicity property is.

\begin{definition}[Acyclicity property]
	Let $Y$ be a marking function and $\xpg$ be a position graph.
	The\emph{ acyclicity property }associated with $Y$ in $\xpg$,  denoted by $Y^X$, is satisfied if 
	there is no marked cycle for an existential position in  $\xpg$. 
         If $Y^X$ is satisfied, we also say that $\XPG{\R}$ \emph{satisfies} $Y$.
\end{definition}

For instance, the marking function associated with weak-acyclicity assigns to each node the set of its successors in $\FPG{\R}$, without any additional constraint.
The next proposition states that such marking functions can be defined for each class of rules between $\wa$ and $\swa$
 (first column in Figure~\ref{fig:gen}), in such a way that the associated acyclicity property in $PG^F$ characterizes this class. 

\begin{proposition}
\label{prop:markings}
	A set of rules $\R$ is $\wa$  (resp.   $\fd$,  $\ar$,  $\ja$,  $\swa$) iff $\FPG{\R}$ satisfies the acyclicity property
	associated with  $\wa$- (resp.  $\fd$-,  $\ar$-,  $\ja$-,  $\swa$-) marking.
\end{proposition}


As already mentioned, all these classes can be safely extended by combining them with the GRD.
To formalize this, we recall the notion $Y^<$ from \cite{grau2013acyclicity}: given an acyclicity property $Y$, a set of rules $\R$ is said to satisfy $Y^<$ if each s.c.c. of $GRD(\R)$ satisfies $Y$, except for those composed of a single rule with no loop.\footnote{This particular case is to cover \emph{aGRD}, in which each s.c.c. is an isolated node.}
Whether $\R$ satisfies $Y^<$ can be checked on $\DPG{\R}$:

\begin{proposition}
\label{prop:yd-eq-scc}
	Let $\R$ be a set of rules, and $Y$ be an acyclicity property.
	$\R$ satisfies $Y^<$ iff $\DPG{\R}$ satisfies $Y$,
	i.e., $Y^< = \grd{Y}$.
\end{proposition}

For the sake of brevity, if $Y_1$ and $Y_2$ are two acyclicity properties,
we write $Y_1 \subseteq Y_2$ if any set of rules satisfying $Y_1$ also
satisfies $Y_2$. The following results are straightforward.

\begin{proposition}
	Let $Y_1,Y_2$ be two acyclicity properties.
	If $Y_1 \subseteq Y_2$,
	then $\grd{Y_1} \subseteq \grd{Y_2}$.
\end{proposition}

\begin{proposition}
\label{prop:y-strict-ygrd}
	Let $Y$ be an acyclicity property.
	If $\agrd \nsubseteq Y$ then $Y \subset \grd{Y}$.
\end{proposition}

Hence, any class of rules  satisfying a property $Y^D$ strictly includes both $\agrd$ and the class characterized by $Y$; 
(e.g., Figure~\ref{fig:gen}, from Column 1 to Column 2). More generally, strict inclusion in the first column leads to strict inclusion in the second one:

\begin{proposition}
\label{prop:prop6}
	Let $Y_1, Y_2$ be two acyclicity properties such that $Y_1 \subset Y_2$, $\wa \subseteq Y_1$ and  $Y_2 \nsubseteq \grd{Y_1}$. 
	Then $\grd{Y_1} \subset \grd{Y_2}$.
\end{proposition}

The next theorem states that $\upg$ is strictly more powerful than $\dpg$; moreover, 
	the ``jump" from $\grd{Y}$ to $\gerd{Y}$  is at least as large as
	from $Y$ to $\grd{Y}$.

\begin{theorem}
\label{theo:ygrd-strict-ygerd}
	Let $Y$ be an acyclicity property.
	If $Y \subset \grd{Y}$ then $\grd{Y} \subset \gerd{Y}$.
	Furthermore, there is an injective mapping from the sets of rules satisfying $\grd{Y}$ but not $Y$, to
	the sets of rules satisfying $\gerd{Y}$ but not $\grd{Y}$.
\end{theorem}

\begin{proof}
	Assume $Y \subset \grd{Y}$
	and $\R$ satisfies $\grd{Y}$ but  not $Y$.
	$\R$ can be rewritten into $\R'$ by applying the following steps.
	First, for each rule $R_i = B_i[\X,\Y] \rightarrow H_i[\Y,\Z] \in \R$,
	let $R_{i,1} = B_i[\X,\Y] \rightarrow p_i(\X,\Y)$ where $p_i$ is a fresh
	predicate; and $R_{i,2} = p_i(\X,\Y) \rightarrow H_i[\Y,\Z]$.
	Then, for each rule $R_{i,1}$, let $R'_{i,1}$ be the rule $(B'_{i,1} \rightarrow H_{i,1})$
	with $B'_{i,1} = B_{i,1} \cup \{p'_{j,i}(x_{j,i}): \forall R_{j} \in \R\}$,
	where $p'_{j,i}$ are fresh predicates and $x_{j,i}$ fresh variables.
	Now, for each rule $R_{i,2}$, let $R'_{i,2}$ be the rule $(B_{i,2} \rightarrow H'_{i,2})$
	with $H'_{i,2} = H_{i,2} \cup \{p'_{i,j}(z_{i,j}): \forall R_{j} \in \R\}$,
	where $z_{i,j}$ are fresh existential variables.
	Let $\R' = \bigcup\limits_{R_i \in \R} \{R'_{i,1},R'_{i,2}\}$.
    This construction ensures that each  $R'_{i,2}$ depends on $R'_{i,1}$,
	and each $R'_{i,1}$ depends on each $R'_{j,2}$,
	thus, there is a {\em transition} edge from each $R'_{i,1}$ to $R'_{i,2}$
	and from each $R'_{j,2}$ to each $R'_{i,1}$.
	Hence, $\DPG{\R'}$ contains exactly one cycle for each cycle in $\FPG{\R}$.
	 Furthermore, $\DPG{\R'}$ contains at least one marked cycle w.r.t. $Y$,
	and then $\R'$ does not satisfy $\grd{Y}$.
	Now, each cycle in $\UPG{\R'}$ 
	is also a cycle in $\DPG{\R}$, and, since $\DPG{\R}$ satisfies $Y$, $\UPG{\R'}$ also does.
    Hence, $\R'$ does not belong to $\grd{Y}$ but to $\gerd{Y}$.
\end{proof}

We also check that strict inclusions in the second column in Figure~\ref{fig:gen} lead to strict inclusions in the third column. 

\begin{theorem} \label{theo:ygrd-strict-ygerd2}
	Let $Y_1$ and $Y_2$ be two acyclicity properties.
	If $\grd{Y_1} \subset \grd{Y_2}$ then $\gerd{Y_1} \subset \gerd{Y_2}$.
\end{theorem}
\begin{proof}
	Let $\R$ be a set of rules such that $\R$ satisfies $\grd{Y_2}$ but 
	does not satisfy $\grd{Y_1}$.
	We rewrite $\R$ into $\R'$ by applying the following steps.
	For each pair of rules $R_i,R_j \in \R$ such that there is a dependency path from $R_i$ to $R_j$,
	for each variable $x$ in the frontier of $R_j$ and each variable $y$ in the
	head of $R_i$, if $x$ and $y$ occur both in a given predicate position,
	we add to the body of $R_j$ a new atom $p_{i,j,x,y}(x)$ and to the head of $R_i$
	a new atom $p_{i,j,x,y}(y)$, where $p_{i,j,x,y}$ denotes a fresh predicate.
	This construction allows each term from the head of $R_i$ to propagate to each
	term from the body of $R_j$, if they share some predicate position in $\R$.
	Thus, any cycle in $\DPG{\R}$ is also in $\UPG{\R'}$, without any change in the
	behavior w.r.t. the acyclicity properties.
	Hence $\R'$ satisfies $\gerd{Y_2}$ but does not satisfy $\gerd{Y_1}$.
\end{proof}

The next result states that  $Y^U$ is a sufficient condition for chase termination:

\begin{theorem} \label{theo:correct}
Let $Y$ be an acyclicity property ensuring the halting of some chase variant $C$. 
Then, the $C$-chase halts for any set of rules $\mathcal R$ that satisfies $Y^U$ (hence $Y^D$). 	
\end{theorem}

\begin{example} Consider again the set of rules $\R$ from Example \ref{ex:dpg-upg}. Figure \ref{fig:dpg-upg}  pictures
 the associated position graphs $\DPG{\R}$ and $\UPG{\R}$. $\R$ is not \emph{aGRD}, nor \emph{wa},  
 nor \emph{wa}$^D$ since $\DPG{\R}$ contains a (marked) cycle that goes through the existential position $\vpos{t(v,w)}{2}$. 
 However, $\R$ is obviously \emph{wa}$^U$ since $\UPG{\R}$ is acyclic. Hence, the skolem chase and stronger chase variants
 halt for  $\R$ and any set of facts.

\end{example}

Finally, we remind that classes from $\wa$ to $\swa$ can be recognized
in PTIME, and checking $\agrd$ is \coNPc. 
Hence, as stated by the next result, the expressiveness gain is without increasing worst-case 
complexity.

\begin{theorem}[Complexity]
\label{theo:complexity}
	Let $Y$ be an acyclicity property, and $\R$ be a set of rules.
	If checking that $\R$ satisfies $Y$ is in \coNP, then
	checking that $\R$ satisfies $\grd{Y}$ or $\gerd{Y}$ is \coNPc.
\end{theorem}

%% file: src/figures/expg2.tex



\begin{gerd:figure}[scale=0.75,every node/.style={transform shape}]

	\gerdRule{0}{0}{4}{3}
		{h/{x}/1}
		{p/{x,y}/1,p/{x,y}/2}
		{h{x}1/p{x,y}1}
		{p{x,y}2}

	\gerdRule{5.5}{0}{9.5}{3}
		{p/{u,v}/1,p/{u,v}/2,q/{v}/1}
		{h/{v}/1}
		{p{u,v}2/h{v}1,q{v}1/h{v}1}
		{}

	\gerdUnif[dashed]{1}{2}{p{x,y}1/p{u,v}1,p{x,y}2/p{u,v}2}{}

	\draw[gerd:edge-unif] (th2_h{v}1) to[out=90,in=0] (4.65,0.5) to[out=180,in=90] (tb1_h{x}1);


%
%
%
%
\end{gerd:figure}

%% file: src/figures/expgS.tex


\begin{gerd:figure}[scale=0.75,every node/.style={transform shape}]

	\gerdRule{0}{0}{4}{3}
		{t/{x,y}/1,t/{x,y}/2}
		{p/{z,y}/1,p/{z,y}/2,q/{y}/1}
		{t{x,y}2/p{z,y}2,t{x,y}2/q{y}1}
		{p{z,y}1}

	\gerdRule{5.5}{0}{9.5}{3}
		{p/{u,v}/1,p/{u,v}/2,q/{u}/1}
		{t/{v,w}/1,t/{v,w}/2}
		{p{u,v}2/t{v,w}1}
		{t{v,w}2}

	\gerdUnif[dashed]{1}{2}{p{z,y}1/p{u,v}1, p{z,y}2/p{u,v}2}{}
	\gerdUnif{1}{2}{q{y}1/q{u}1}{}

	\draw[gerd:edge-unif] (th2_t{v,w}1) to[out=110,in=0] (4.65,0.75) to[out=180,in=70] (tb1_t{x,y}1);

	\draw[gerd:edge-unif] (th2_t{v,w}2) to[out=-110,in=0] (4.65,-3.75) to[out=180,in=-70] (tb1_t{x,y}2);
%
%


\end{gerd:figure}

%% file: src/content/unif.tex

In this section, we show how to further extend $Y^U$ into $Y^{U^+}$ 
by a finer analysis of marked cycles and unifiers. 
This extension can be performed without increasing complexity.  
We define the notion of \emph{incompatible} sequence of unifiers, 
which ensures that a given sequence of rule applications is impossible. 
Briefly, a marked cycle for which all sequences of unifiers are incompatible
 can be ignored. 
 Beside the gain for positive rules,
this refinement will allow one to  take better advantage of negation. 

We first point out that the notion of piece-unifier is not appropriate to our purpose.
We have to relax it, as illustrated by the next example. 
We call \emph{unifier}, of a rule body $B_2$ with a
rule head $H_1$, a substitution $\mu$ of $\fun{vars}(B'_2) \cup \fun{vars}(H'_1)$, 
where $B'_2 \subseteq B_2$ and $H'_1
\subseteq H_1$, such that  $\mu(B'_2) = \mu(H'_1)$ 
(thus, it satisfies Condition $(1)$ of a piece-unifier). 

\begin{example} \label{ex-unif} Let $\mathcal R = \{R_1, R_2, R_3, R_4\}$ with:\\
$R_1 : p(x_1,y_1) \rightarrow q(y_1,z_1)$$~~~$\\
$R_2 : q(x_2,y_2) \rightarrow r(x_2,y_2)$\\
$R_3 : r(x_3,y_3) \wedge s(x_3,y_3) \rightarrow p(x_3,y_3)$ \\ $R_4 : q(x_4,y_4) \rightarrow s(x_4,y_4)$\\ 
There is a dependency cycle $(R_1, R_2, R_3, R_1)$ and a corresponding cycle in $PG^U$. 
We want to know if such a sequence of rule applications is possible. 
We build the following new rule, which is a composition of $R_1$ and $R_2$ (formally defined later): 
$R_1 \unified_\unifier R_2 : p(x_1,y_1) \rightarrow q(y_1,z_1) \wedge r(y_1,z_1)$\\
There is no piece-unifier of $R_3$ with $R_1 \unified_\unifier R_2$, since $y_3$ would be a separating variable mapped to the existential variable $z_1$. 
This actually means that $R_3$ is not applicable \emph{right after} $R_1 \unified_\unifier R_2$. 
However, the atom needed to apply $s(x_3,y_3)$ can be brought by a sequence of rule applications $(R_1,R_4)$. 
We thus relax the notion of piece-unifier to take into account arbitrarily long sequences of rule applications. 
\end{example}

\begin{definition}[Compatible unifier]
	Let $R_1$ and $R_2$ be rules.
	A unifier $\unifier$ of $B_2$ with $H_1$ is {\em compatible}
	if, for each position $\vpos{a}{i}$ in $B'_2$, 
	such that $\unifier(\fun{term}(\vpos{a}{i}))$ is an existential
	variable $z$ in $H'_1$, $\UPG{\R}$ contains a path,
       from a position in which $z$ occurs, 
	to $\vpos{a}{i}$, that does not go through another existential position. Otherwise, $\unifier$ is \emph{incompatible}. 
\end{definition}

Note that a piece-unifier is necessarily compatible.

\begin{proposition}
\label{prop:inc-unif-noapp}
Let $R_1$ and $R_2$ be rules, and let $\mu$ be a unifier of $B_2$ with $H_1$. If $\mu$ is incompatible, then no application of $R_2$ can use an atom in $\mu(H_1)$. 
\end{proposition}

We define the rule corresponding to the composition of $R_1$ and $R_2$ according to a  
compatible unifier, then use this notion to define a compatible sequence of unifiers. 

\begin{definition}[Unified rule, Compatible sequence of unifiers]
\label{def:unified-rule}$~$\\
$\bullet$ Let $R_1$ and $R_2$ be rules such that there is a compatible unifier $\unifier$
	 of $B_2$ with $H_1$.
	The associated {\em unified rule} $R_\unifier = R_1 \unified_\unifier R_2$ is defined by $H_\unifier = \unifier(H_1) \cup \unifier(H_2)$, 
		 and $B_\unifier = \unifier(B_1) \cup (\unifier(B_2) \setminus \unifier(H_1))$.\\
	$\bullet$ Let $(R_1, \ldots, R_{k+1})$ be a sequence of rules. A sequence $s = (R_1~\mu_1 ~R_2 \ldots ~\mu_{k}~R_{k+1})$, 
	where, for $1 \leq i \leq k$,  $\unifier_i$ is a unifier of $B_{i+1}$ with $H_i$, 
	 is  a \emph{compatible sequence} of unifiers if: \emph{(1) } $\mu_1$ is a compatible unifier of $B_2$ with $H_1$, and\emph{ (2) }
	 if $k > 0$, the sequence obtained from $s$ by replacing $(R_1 ~\mu_1~R_2)$ with 
	$R_1 \unified_{\unifier_1} R_{2}$ is a \emph{compatible} sequence of unifiers.
\end{definition}

E.g., in Example~\ref{ex-unif}, the sequence $(R_1~\mu_1~R_2~\mu_2~R_3~\mu_3~R_1)$, with the obvious $\mu_i$, is compatible. 
 We can now improve all previous acyclicity properties (see the fourth column in Figure~\ref{fig:gen}).

\begin{definition}[Compatible cycles]
	Let $Y$ be an acyclicity property, and $\upg$ be a position graph with unifiers.
	The {\em compatible cycles for $\vpos{a}{i}$} in $\upg$ are all marked cycles $C$ for
	$\vpos{a}{i}$ w.r.t. $Y$,
	such that there is a compatible sequence of unifiers induced by $C$.
	Property $\gerdc{Y}$is satisfied if, for each existential position $\vpos{a}{i}$, there is no compatible cycle for $\vpos{a}{i}$ in $\upg$.
\end{definition}

Results similar to Theorem \ref{theo:ygrd-strict-ygerd} and Theorem  \ref{theo:ygrd-strict-ygerd2} are obtained for $Y^{U^+}$ w.r.t. $Y^U$, namely:
\begin{itemize}
\item For any acyclicity property $Y$, $\gerd{Y} \subset \gerdc{Y}$.
\item For any acyclicity properties $Y_1$ and $Y_2$, if $\gerd{Y_1} \subset \gerd{Y_2}$, then $\gerdc{Y_1} \subset \gerdc{Y_2}$.
\end{itemize}

Moreover, Theorem \ref{theo:correct} can be extended to $Y^{U^+}$:  let $Y$ be an acyclicity property ensuring the halting of some chase variant $C$; 
then the $C$-chase halts for any set of rules $\mathcal R$ that satisfies $Y^{U^+}$ (hence $Y^U$). 
Finally, the complexity result from Theorem \ref{theo:complexity} still holds
 for this improvement.

%% file: src/content/negation.tex


We now add nonmonotonic negation, which we denote by \textbf{not}. A \emph{nonmonotonic existential} (NME) rule $R$ is of the form 
 $\forall \vec{x}\forall \vec{y} (B^+ \wedge \fun{\bf not}B^-_1 \wedge \ldots \wedge \fun{\bf not}B^-_k \rightarrow \exists \vec{z} H)$, where $B^+$, $B^-=\{B^-_1 \ldots B^-_k\}$ and $H$ are atomsets, respectively called the \emph{positive} body, the \emph{negative} body and the head of $R$; furthermore, $\fun{vars}(B^-) \subseteq \fun{vars}(B^+)$. 
$R$ is \emph{applicable} to $F$ if there is a homomorphism  $h$ from $B^+$ to $F$ such that $h(B^-) \cap F = \emptyset$. In this section, we rely on a skolemization of the knowledge base. Then, the application of $R$ to $F$ w.r.t. $h$ produces $h(sk(H))$. 
$R$ is  {\em self-blocking}   if $B^- \cap (B^+ \cup H) \neq \emptyset$, i.e., $R$ is never applicable.

Since skolemized NME rules can be seen as normal logic programs, we can rely on the standard definition of stable models  \cite{gl88}, which we omit here since it is not needed to understand the sequel. Indeed, our acyclicity criteria essentially ensure that there is a finite number of skolemized rule applications. 
Although the usual definition of stable models relies on grounding (i.e., instantiating) skolemized rules, stable models of $(F,\mathcal R)$ can be computed by a skolem chase-like procedure, as performed by Answer Set Programming solvers that instantiate rules  on the fly \cite{asperix09,dao2012omiga}.

We check that, when the skolem chase halts on the positive part of NME rules (i.e., obtained by ignoring the negative body),  the stable computation based on the skolem chase halts. 
 We can thus rely on preceding acyclicity conditions, which already generalize known acyclicity conditions applicable to skolemized NME rules (for instance \emph{finite-domain} and \emph{argument-restricted}, which were defined for normal logic programs). 
 We can also extend them by exploiting negation. 
 
First, we consider the natural extensions of a unified rule (Def.~\ref{def:unified-rule}) and of rule dependency: to define $R_\unifier = R_1 \unified_\unifier R_2$, we add that $B_\unifier^- = \mu(B_1^-) \cup \mu(B_2^-)$; besides, $R_2$ depends on $R_1$ if there is a piece-unifier $\mu$ of $H_2$ with $B_1$ such that $R_1 \unified_\unifier R_2$ is not self-blocking; if $R_1 \unified_\unifier R_2$ is self-blocking, we say that  $\mu$ is self-blocking. Note that this extended dependency is equivalent to the \emph{positive reliance} from \cite{mkh13}. In this latter paper, positive reliance is used to define an acyclicity condition: a set of NME rules is said to be \emph{ R-acyclic} if no cycle of positive reliance involves a rule with an existential variable. Consider now $PG^D$  with extended dependency: then, R-acyclicity is stronger than aGRD (since cycles are allowed on rules without existential variables) but weaker than $wa^D$ (since all s.c.c. are necessarily wa). 
 
  By considering extended dependency, we can extend the results obtained with $\dpg$ and $\upg$ 
(note that for $\upg$ we only encode non-self-blocking unifiers). 
  We can further extend $Y^{U+}$ classes by considering \emph{self-blocking compatible sequences} of unifiers. 
Let $C$ be a compatible cycle for $\vpos{a}{i}$ in  $PG^U$, and  
$C_\mu$ be the set of all compatible sequences of unifiers induced by $C$. 
A sequence $\mu_1 \ldots \mu_k \in \mathcal C_\mu$ is said to be self-blocking if the rule
$R_1 \unified_{\mu_1} R_2 \ldots R_k \unified_{\mu_k} R_{1}$ is self-blocking.
When all sequences in $C_\mu$ are self-blocking, $C$ is said to be self-blocking.


\begin{example}
\label{ex:selfblock-unifier}
Let $R_1 = q(x_1), {\bf not} p(x_1) \rightarrow r(x_1,y_1)$, $R_2 =  r(x_2,y_2) \rightarrow s(x_2,y_2)$, 
$R_3 = s(x_3,y_3) \rightarrow p(x_3), q(y_3)$.
$PG^{U+}(\{R_1,R_2, R_3\})$ has a unique cycle, with a unique induced  compatible unifier sequence. The rule 
$R_1 \unified R_2 \unified R_3 = q(x_1), {\bf not} p(x_1) \rightarrow r(x_1,y_1), s(x_1,y_1), p(x_1), q(y_1)$  
 is self-blocking, hence $R_1 \unified R_2 \unified R_3 \unified R_1$ also is. Thus,  there is no ``dangerous" cycle.
\end{example}


\begin{proposition}
\label{prop:neg-correct}
If, for each existential position $\vpos{a}{i}$, all compatible cycles for $\vpos{a}{i}$ in $PG^U$ are self-blocking, then
the stable computation based on the skolem chase halts.
\end{proposition}

Finally, we point out that these improvements do not increase worst-case complexity of the acyclicity test. 

%% file: src/content/conclusion.tex

We have proposed a tool that allows to unify and generalize most existing acyclicity conditions for existential rules, without increasing worst-case complexity. This tool can be further refined to deal with nonmonotonic (skolemized) existential rules, which, to the best of our knowledge, extends all known acyclicity conditions for this kind of rules. 

Further work includes the implementation of the tool\footnote{It will be developed as an extension of KIABORA, an analyzer of existential rule bases \cite{lmr13}.} and experiments on real-world ontologies, as well as the study of chase variants that would allow to process existential rules with stable negation without skolemization.

%% file: src/content/appendix.tex

\input{src/content/appendix-prop2.tex}

\paragraph{Proposition \ref{prop:prop6}}
\emph{
Let $Y_1$ and $Y_2$ be two acyclicity properties such that
$Y_1 \subset Y_2$, $wa \subseteq Y_1$ and $Y_2 \nsubseteq Y_1^D$.
Then $Y_1^D \subset Y_2^D$.
}

\begin{proof}
	Let $\R$ be a set of rules such that $\R$ satisfies $Y_2$ and neither $Y_1$ nor $\agrd$.
	$\R$ can be rewritten into $\R'$ by replacing each rule $R_i = (B_i,H_i) \in \R$ with a new
	rule $R'_i = (B_i \cup \{p(x)\}, H_i \cup \{p(x)\})$ where $p$ is a fresh predicate and
	$x$ a fresh variable.
	Each rule can now be unified with each rule, but the only created cycles are  those which contain
	only atoms $p(x)$, hence none of those cycles go through existential positions.
	Since $\wa \subseteq Y_1$ (and so $\wa \subseteq Y_2$), the added cycles do not change the behavior
	of $\R$ w.r.t. $Y_1$ and $Y_2$.
	Hence,  $\R'$ is a set of rules satisfying $Y_2$ and not $Y_1$, and since $GRD(\R')$ is a complete
	graph, $\DPG{\R'} = \FPG{\R'}$.
	We can conclude that $\R'$ satisfies $\grd{Y_2}$ but not $\grd{Y_1}$.
\end{proof}


\paragraph{Theorem \ref{theo:ygrd-strict-ygerd}}
\emph{
Let $Y$ be an acyclicity property.
If $Y \subset \grd{Y}$ then $\grd{Y} \subset \gerd{Y}$.
Furthermore, there is an injective mapping from the sets of rules satisfying $\grd{Y}$ but not $Y$, to
the sets of rules satisfying $\gerd{Y}$ but not $\grd{Y}$.
}

\begin{proof} \textbf{(included in the paper)}
	Assume $Y \subset \grd{Y}$
	and $\R$ satisfies $\grd{Y}$ but  not $Y$.
	$\R$ can be rewritten into $\R'$ by applying the following steps.
	First, for each rule $R_i = B_i[\X,\Y] \rightarrow H_i[\Y,\Z] \in \R$,
	let $R_{i,1} = B_i[\X,\Y] \rightarrow p_i(\X,\Y)$ where $p_i$ is a fresh
	predicate ; and $R_{i,2} = p_i(\X,\Y) \rightarrow H_i[\Y,\Z]$.
	Then, for each rule $R_{i,1}$, let $R'_{i,1}$ be the rule $(B'_{i,1} \rightarrow H_{i,1})$
	with $B'_{i,1} = B_{i,1} \cup \{p'_{j,i}(x_{j,i}) : \forall R_{j} \in \R\}$,
	where $p'_{j,i}$ are fresh predicates and $x_{j,i}$ fresh variables.
	Now, for each rule $R_{i,2}$ let $R'_{i,2}$ be the rule $(B_{i,2} \rightarrow H'_{i,2})$
	with $H'_{i,2} = H_{i,2} \cup \{p'_{i,j}(z_{i,j}) : \forall R_{j} \in \R\}$,
	where $z_{i,j}$ are fresh existential variables.
	Let $\R' = \bigcup\limits_{R_i \in \R} \{R'_{i,1},R'_{i,2}\}$.
    This construction ensures that each  $R'_{i,2}$ depends on $R'_{i,1}$,
	and each $R'_{i,1}$ depends on each $R'_{j,2}$,
	thus, there is a {\em transition} edge from each $R'_{i,1}$ to $R'_{i,2}$
	and from each $R'_{j,2}$ to each $R'_{i,1}$.
	Hence, $\DPG{\R'}$ contains exactly one cycle for each cycle in $\FPG{\R}$. 
	 Furthermore, $\DPG{\R'}$ contains at least one marked cycle w.r.t. $Y$,
	and then $\R'$ is not $\grd{Y}$.
	Now, each cycle in $\UPG{\R'}$ 
	is also a cycle in $\DPG{\R}$, and since $\DPG{\R}$ satisfies $Y$, $\UPG{\R'}$ also does.
    Hence, $\R'$ does not belong to $\grd{Y}$ but to $\gerd{Y}$.
\end{proof}

\paragraph{Theorem \ref{theo:ygrd-strict-ygerd2}}
\emph{
Let $Y_1$ and $Y_2$ be two acyclicity properties.
If $\grd{Y_1} \subset \grd{Y_2}$ then $\gerd{Y_1} \subset \gerd{Y_2}$.
}

\begin{proof} \textbf{(included in the paper)}
	Let $\R$ be a set of rules such that $\R$ satisfies $\grd{Y_2}$ but 
	does not satisfy $\grd{Y_1}$.
	We rewrite $\R$ into $\R'$ by applying the following steps.
	For each pair of rules $R_i,R_j \in \R$ such that $R_j$ depends on $R_i$,
	for each variable $x$ in the frontier of $R_j$ and each variable $y$ in the
	head of $R_i$, if $x$ and $y$ occur both in a given predicate position,
	we add to the body of $R_j$ a new atom $p_{i,j,x,y}(x)$ and to the head of $R_i$
	a new atom $p_{i,j,x,y}(y)$, where $p_{i,j,x,y}$ denotes a fresh predicate.
	This construction will allow each term from the head of $R_i$ to propagate to each
	term from the body of $R_j$, if they shared some predicate position in $\R$.
	Thus, any cycle in $\DPG{\R}$ is also in $\UPG{\R'}$, without modifying
	behavior w.r.t. the acyclicity properties.
	Hence, $\R'$ satisfies $\gerd{Y_2}$ but does not satisfy $\gerd{Y_1}$.
\end{proof}

\paragraph{Theorem \ref{theo:correct}}
\emph{
Let $Y$ be an acyclicity property ensuring the halting of some chase variant $C$.
Then the $C$-chase halts for any set of rules $\R$ that satisfies $\gerd{Y}$ (hence $\grd{Y}$).
}

We will first formalize the notion of a \emph{correct} position graph 
(this notion being not precisely defined in the core paper). 
Then, we will prove that $PG^F$,  $PG^D$ and $PG^U$ are correct, which will allow to prove the theorem.  

\paragraph{Preliminary definitions}

Let $F$ be a fact and $\mathcal R$ be a set of rules. An $\mathcal R$-\emph{derivation} (sequence) (from $F$ to $F_k$) is a finite sequence $(F_0=F),(R_1, \pi_1, F_1), \ldots, (R_{k}, \pi_{k}, F_k)$ such that  for all $0 < i \leq k$, $R_i \in \mathcal R$ and $\pi_i$ is a homomorphism from $\fun{body}({R_i})$ to $F_{i-1}$ such that  $F_{i} = \alpha(F_{i-1}, R_i, \pi_i)$. When only the successive facts are needed, we note $(F_0=F),F_1, \ldots, F_k$.

Let $S= (F_0=F), \ldots, F_n$ be a breadth-first  $\mathcal R$-derivation from $F$. \footnote{A derivation is breadth-first if, given any fact  $F_i$ in the sequence, all rule applications corresponding to homomorphisms to $F_i$ are performed before rule applications on subsequently derived facts that do not correspond to homomorphisms to $F_i$.}
Let $h$ be an atom in the head of $R_i$ and $b$ be an atom in the body of $R_j$.
We say that $(h, \pi_i)$ is a \emph{support} of $(b, \pi_j)$ (in $S$) if $\pi_i^{safe}(h) = \pi_j(b)$. We also say that an atom $f \in F_0$ is a \emph{support} of $(b, \pi_j)$ if $\pi_j(b) = f$. In that case, we note $(f, \mbox{\sl init})$ is a support of $(b, \pi_j)$. Among all possible supports for $(b, \pi_j)$, its \emph{first supports} are the $(h, \pi_i)$ such that $i$ is minimal or $\pi_i = \mbox{\sl init}$. Note that $(b, \pi_j)$ can have two distinct first supports $(h, \pi_i)$ and $(h', \pi_i)$ when the body of $R_i$ contains two distinct atoms $h$ and $h'$ such that $\pi_i^{safe}(h) = \pi_i^{safe}(h')$. By extension, we say that $(R_i, \pi_i)$ is a \emph{support} of $(R_j, p_j)$ in $S$ when there exist an atom $h$ in the head of $R_i$ and an atom $b$ in the body of $R_j$ such that $(h, \pi_i)$ is a first support of $(b, \pi_j)$. In the same way, $F_0$ is a support of $(R_j, \pi_j)$ when there exists $b$ in the body of $R_j$ such that $\pi_j(b) \in F_0$. Among all possible supports for $(R_j, \pi_j)$, its \emph{last support} is the support $(R_i, \pi_i)$ such that $i$ is maximal.

The \emph{support graph} of $S$ has $n+1$ nodes, $F_0$ and the $(R_i, \pi_i)$. We add an edge from $I = (R_i, \pi_i)$ to $J = (R_j, \pi_j)$ when $I$ is a support of $J$. Such an edge is called a \emph{last support edge} (LS edge) when $I$ is a last support of $J$. An edge that is not LS is called \emph{non transitive} (NT) if it is not a transitivity edge. A path in which all edges are either LS or NT is called a \emph{triggering path}.

\begin{definition}[Triggering derivation sequence]  A \emph{$h \rightarrow b$ triggering derivation sequence} is a breadth-first derivation sequence 
$F = F_0, \ldots, F_n$ from $F$ where $(h, \pi_1)$ is a first support of $(b, \pi_n)$.
\end{definition}

\begin{definition}[Correct position graph] Let $\mathcal R$ be a set of rules. A position graph of $\mathcal R$ is said to be \emph{correct} if, whenever there exists a \emph{$h \rightarrow b$ triggering derivation sequence}, the position graph contains a transition from $[h,i]$ to $[b,i]$ for all $1 \leq i \leq k$, where $k$ is the arity of the predicate of $h$ and $b$.
\end{definition}

\begin{proposition} $PG^F$ is correct.
\end{proposition}

\begin{proof} Follows from the above definitions. 
\end{proof}

\begin{lemma}\label{support} If $S$ is a $h \rightarrow b$ triggering derivation sequence, then there is a triggering path from $(R_1, \pi_1)$ to $(R_n, \pi_n)$ in the support graph of $S$.
\end{lemma}

\begin{proof} There is an edge from $(R_1, \pi_1)$ to $(R_n, \pi_n)$ in the support graph of $S$. By removing transitivity edges, it remains a path from $(R_1, \pi_1)$ to $(R_n, \pi_n)$ for which all edges are either LS or NT.
\end{proof}

\begin{lemma} \label{support2grd} If there is an edge from $(R_i, \pi_i)$ to $(R_j, \pi_j)$ that is either LS or NT in the support graph of $S$, then $R_j$ depends on $R_i$.
\end{lemma}

\begin{proof} Assume there is a LS edge from $(R_i, \pi_i)$ to $(R_j, \pi_j)$ in the support graph. Then the application of $R_i$ according to $\pi_i$ on $F_{i-1}$ produces $F_{i}$ on which all atoms required to map $B_j$ are present (or it would not have been a last support). Since it is a support, there is also an atom required to map $B_j$ that appeared in $F_{i-1}$. It follows that $R_j$ depends upon $R_i$.

Suppose now that the edge is NT. Consider $F_k$ such that there is a LS edge from  $(R_k, \pi_k)$ to $(R_j, \pi_j)$. See that there is no path in the support graph from $(R_i, \pi_i)$ to $(R_k, \pi_k)$ (otherwise there a would be a path from $(R_i, \pi_i)$ to $(R_j, \pi_j)$ and the edge would be a transitive edge). In the same way, there is no $q$ such that there is a path from $(R_i, \pi_i)$ to $(R_j, \pi_j)$ that goes through $(R_q, \pi_q)$ (or the edge from $(R_i, \pi_i)$ to $(R_j, \pi_j)$ would be transitive). Thus, we can consider the atomset $F_{k \setminus i}$ that would have been created by the following derivation sequence:
 
\begin{itemize}
    \item first apply from $F_0$ all rule applications of the initial sequence until $(R_{i-1}, \pi_{i-1})$;
    \item then apply all possible rule applications of this sequence, from $i+1$ until $k$.
\end{itemize}
 We can apply $(R_i, \pi_i)$  on the atomset $F_{k \setminus i}$ thus obtained (since it contains the atoms of $F_{i-1}$). Let us now consider the atomset $G$ obtained after this rule application. We must now check that $(R_j, \pi_j)$ can be applied on $G$: this stems from the fact that there is no support path from $(R_i, \pi_i)$ to $(R_j, \pi_j)$. This last rule application relies upon an atom that is introduced by the application of $(R_i, \pi_i)$, thus $R_j$ depends on $R_i$.
\end{proof}

\begin{proposition} $PG^D$ is correct.
\end{proposition}

\begin{proof} If there is a $h \rightarrow b$ triggering derivation sequence, then (by Lemma \ref{support}) we can exhibit a triggering path that corresponds to a path in the GRD (Lemma \ref{support2grd}).
\end{proof}

\begin{proposition} $PG^U$ is correct.
\end{proposition}

\begin{proof} Consider a $h \rightarrow b$ triggering derivation sequence $F = F_0, \ldots, F_n$. We note $H^P = \pi_n(B_n) \cap \pi_1^{safe}(H_1)$ the atoms of $F_n$ that are introduced by the rule application $(R_1, \pi_1)$ and are used for the rule application $(R_n, \pi_n)$. Note that this atomset $H^P$ is not empty, since it contains at least the atom produced from $h$. Now, consider the set of terms $T^P = \fun{terms}(H^P) \cap \fun{terms}(\pi_n(B_n) \setminus H^P)$ that separate the atoms of $H^P$ from the other atoms of $\pi_n(B_n)$.

Now, we consider the rule $R^P = B_1 \cup \{fr(t) \, | \, t$ is a variable of $R_1$ and $\pi_1^{safe}(v) \in T^P$ $ \, \}\rightarrow H_1$. Consider the atomset
$F^P = F_{n-1} \setminus H^P \cup \{fr(t) \, | \, $ t is a term of $T^P\}$.

Consider the mapping $\pi_1^P$ from the variables of the body of $R_P$ to those of $F_P$, defined as follows: if $v$ is a variable of $B_1$, then  $\pi_1^P(v) = \pi_1(v)$, otherwise $v$
is a variable in an ``fr'' atom and  $\pi_1^P(v) = \pi_n(v)$. This mapping is a homomorphism, thus we can consider the atomset ${F^P}' = \alpha(F^P, R^P, \pi_1^P)$. This application
produces a new application of $R_n$ that maps $b$ to the atom produced from $h$. Indeed, consider the mapping $\pi_n^P$ from the variables of $B_n$ to those of ${F^P}'$ defined as
follows: if $t$ is a variable of $B_n$ such that $\pi_n(t) \in \fun{terms}(H^P) \setminus T^P$, then $\pi_n^P(t) = {\pi_1^P}^{safe}(t')$, where $t'$ is the variable of $H_1$ that
produced $\pi_n(t)$, otherwise $\pi_n^P(t) = \pi_n(t)$. This mapping is a homomorphism. This homomorphism is new since it maps $b$ to ${\pi_n^P}^{safe}(h)$. Thus,
 there is a piece-unifier of $B_n$ with the head of $R^P$ that unifies $h$ and $b$.

It remains now to prove that for each atom $fr(t)$ in the body of $R^P$ there exists a triggering path $P_i = (R'_1,\pi'_1) = (R_1,\pi_1)$ to $(R'_k,\pi'_k) = (R_n,\pi_n)$ in the support
graph such that $fr(t)$ appears in the agglomerated rule $R_i^A$ along $R_1,\dots,R_{n-1}$.

Let $t$ be a variable occuring in some $fr$ atom in $R^P$.
Suppose that $fr(t)$ does not appear in any agglomerated rule corresponding to a triggering path $P_i$ between $(R_1,\pi_1)$ and $(R_n,\pi_n)$.
Since $\pi_1(t)$ is an existential variable generated by the application of $R_1$,
and there is no unifier on the GRD paths that correspond to these triggering paths that unify $t$, 
$\pi_1(t)$ may only occur in atoms that are not used (even transitively) by $(R_n,\pi_n)$, i.e. $\pi_1(t) \notin T^P$.
Therefore, $t$ does not appear in a $fr$ atom $R^P$, which leads to a contradiction.

Since $R^P$ and $R^A = \bigcup R^A_i$ have the same head and the frontier of $R^P$ is a subset of the frontier of $R^A$
any unifier with $R^P$ is also a unifier with $R^A$. Thus, there is a unifier of $R_n$ with $R^A$ that unifies $h$ and $b$, and there are the corresponding correct
transition edges in PG$^U$.

\end{proof}

\begin{proof} (of Theorem  \ref{theo:correct})

Let us say that a transition edge from $\vpos{a}{i}$ in $R_1$  to $\vpos{a',i}$ in $R_2$ is \emph{useful} if there is a fact $F$, and a homomorphism $\pi_1$ 
from $B_1$ to $F$, such that there is a homomorphism $\pi_2$ from $B_2$
to some $F'$ obtained from a derivation of $(\alpha(F,R_1,\pi_1),\R)$ and $\pi_1^{safe}(a) = \pi_2(a')$.
Furthermore, we say that the application of $R_2$ \emph{uses} edge $(\vpos{a}{i}, \vpos{a'}{i})$.

One can see that a useful edge exactly corresponds to a $h \rightarrow b$ triggering sequence where
$\vpos{a}{i}$ occurs in $h$ and $\vpos{a}{i'}$ occurs in $b$.
It follows from the correctness of $PG^U$ and $PG^D$ that no useful edge of $PG^F$is removed.

Now, let $Y$ be an acyclicity proposition ensuring the halting of some chase variant $C$.
Assume there is a set of rules $\mathcal R$ that satisfies $Y^U$ but not $Y^D$ and there is $F$ such that the $C$-chase does not halt on $(F, \mathcal R)$. 
Then, there is a rule application in this (infinite) derivation that uses a transition edge $(\vpos{a}{i}, \vpos{a'}{i})$ belonging to $Y^D$ but not $Y^U$.
This is impossible because such an edge is useful. The same arguments hold for $Y^D$ w.r.t. $Y^F$. 
\end{proof}


\paragraph{Theorem \ref{theo:complexity}}
\emph{
Let $Y$ be an acyclicity property, and $\R$ be a set of rules.
If checking that $\R$ is $Y$ is in co-NP, then checking that $\R$ is $\grd{Y}$ or $\gerd{Y}$
is co-NP-complete.
}

We first state a preliminary proposition.

\begin{proposition} 
If there is a $h \rightarrow b$ triggering derivation sequence (with $h \in \fun{head}(R)$ and $b \in \fun{body}(R')$), 
then there exist a non-empty set of paths $\mathcal P = \{P_1,\dots,P_k\}$ from $R$ in GRD($\mathcal R$)
such that $\sum\limits_{1 \leq i \leq k}|P_i| \leq |\R| \times |\fun{terms}(\fun{head}(R))|$
and a piece-unifier of $B'$ with the head of an agglomerated rule along $\mathcal P$ that
unifies $h$ and $b$.  
\end{proposition}

\begin{proof} 
The piece-unifier is entirely determined by the terms that are forced into the frontier by an ``fr'' atom. 
Hence, we need to consider at most one path for each term in $H$. 
Moreover, each (directed) cycle in the GRD (that is of length at most $|\mathcal R|$) needs to be traversed at most $|\fun{terms}(H)|$ times,
 since going through such a cycle without creating a new frontier variable cannot create any new unifier. Hence, we need to consider only
 paths of polynomial length. 
\end{proof}

\vspace*{0.5cm}
\begin{proof} (of Theorem \ref{theo:complexity})
	One can guess a cycle in $PG^D(\R)$ (or $PG^U(\R)$) such that the property $Y$ 
	is not satisfied by this cycle. 	
	From the previous property, each edge of the cycle has a polynomial certificate, 
	and checking if a given substitution is a piece-unifier can also be done in polynomial time.
	Since $Y$ is in co-NP, we have a polynomial certificate that this cycle does not satisfy $Y$.
	Membership to co-NP follows.

	The completeness part is proved by a simple reduction from the co-problem of rule dependency checking (which is thus a co-NP-complete problem).

	Let $R_1$ and $R_2$ be two rules.
	We first define two fresh predicates $p$ and $s$ of arity $|vars(B_1)|$, 
	and two fresh predicates $q$ and $r$ of arity $|vars(H_2)|$.

	We build $R_0 = p(\X) \rightarrow s(\X)$, where $\X$ is a list of all variables in $B_1$, 
	and $R_3 = r(\X) \rightarrow p(\Z),q(\X)$, where $\Z = (z,z,\dots,z)$, where $z$ is a variable which does not appear in $H_2$,
	and $\X$ is a list of all variables in $H_2$.
	We rewrite $R_1$ into $R'_1 = (B_1, s(\X) \rightarrow H_1)$, where $\X$ is a list of all variables in $B_1$, 
	and $R_2$ into $R'_2 = (B_2 \rightarrow H_2 , r(\X))$,  where $\X$ is a list of all variables in $H_2$.
	One can check that $\R = \{R_0,R'_1,R'_2,R_3\}$ contains a cycle going through an existential variable (thus, it is not $wa^D$) iff $R_2$ depends on $R_1$.
\end{proof}

\paragraph{Proposition \ref{prop:inc-unif-noapp}}
\emph{
Let $R_1$ and $R_2$ be rules, and let $\mu$ be a unifier of $B_2$ with $H_1$.
If $\mu$ is incompatible, then no application of $R_2$ can use an atom in $\mu(H_1)$.
}

\begin{proof}
          We first formalize the sentence  ``no application of $R_2$ can use an atom in $\mu(H_1)$'' by the following sentence: 
           ``no application $\pi'$ of $R_2$ can map an atom $a \in B_2$ to an  atom $b$ 
           produced by a application $(R_1, \pi)$
	such that $ b = \pi(b')$, where $\pi$ and $\pi'$ are more specific than $\mu$'' (given two substitutions $s_1$ and $s_2$, $s_1$ is more specific than $s_2$ if there is a substitution   $s$ such that $s_1 = s \circ s_2$). 
	
	Consider the application of $R_1$ to a fact w.r.t. a homomorphism $\pi$, followed by 
	an application of $R_2$  w.r.t. a homomorphism $\pi'$,  
	such that for an atom
	$a \in B_2$, $\pi'(a) = b =  \pi(b')$, where $\pi$ and $\pi'$ are more specific than $\mu$. 
	Note that this implies that $\mu(a) = \mu(b')$. 
          	Assume that $b$ contains a fresh variable $z_i$ produced from an existential variable $z$ of $b'$ in $H_1$. 
	Let $z'$ be the variable from $a$ such that $\pi'(z') = z_i$. 
	Since the domain of $\pi'$ is $B_2$, 
	all atoms from $B_2$ in which $z'$ occurs at a given position $p_j$ are also mapped by $\pi'$ 
	to atoms containing $z_i$ in the same position $p_j$. Since $z_i$ is a fresh variable,
	these atoms have been produced by sequences of rule applications starting from $(R_1, \pi)$.
	Such a sequence of rule applications exists only if there is a path in $PG^U$ from a position of $z$ in $H_1$
	to $p_j$; moreover, this path cannot go through an existential position, otherwise $z_i$ cannot be propagated. 
	Hence, $\mu$ is necessarily compatible. 
	
\end{proof}

\paragraph{Proposition \ref{prop:neg-correct}}
\emph{
If, for each each existential position $\vpos{a}{i}$, all compatible cycles for $\vpos{a}{i}$
in $PG^U$ are self-blocking, then the stable computation based on the skolem chase halts.
}

\begin{proof}
If a cycle is non-compatible or self-blocking, then no sequence of rule applications can use it (where "used" is defined as in the proof of Theorem \ref{theo:correct}). 
Hence, if all compatible cycles are self-blocking, all derivations obtained with skolemized NME rule applications are finite. Hence, the stable computation based on the skolem chase halts. 

\end{proof}

%% file: src/content/appendix-prop2.tex


\paragraph{Proposition \ref{prop:markings}}
\emph{
	A set of rules $\R$ is $\wa$  (resp.   $\fd$,  $\ar$,  $\ja$,  $\swa$) iff $\FPG{\R}$ satisfies the acyclicity property
	associated with the $\wa$-marking (resp.  $\fd$-,  $\ar$-,  $\ja$-,  $\swa$-marking).
}

\vspace*{0.3cm}
To prove Proposition \ref{prop:markings} we rely on some intermediary results. The next proposition is immediate. 

\begin{proposition}
	\label{prop:pg-fpg}
	For each edge $(p_i,q_j)$ in the predicate position graph of a set of rules $\R$,
	there is the following non-empty set of edges in $\FPG{\R}$:
	$E_{p_i,q_j} = \{ (\vpos{a}{i}, \vpos{a'}{j})\ |\ \fun{pred}(\vpos{a}{i}) = p$ and $\fun{pred}(\vpos{a'}{j}) = q \}$. 

	Furthermore, these sets of edges form a partition of all edges in $\FPG{\R}$.
\end{proposition}

We now define marking functions,  whose associated acyclicity property
corresponds to  $\wa$,  $\fd$,  $\ar$,  $\ja$ or  $\swa$ when it is applied on $\FPG{\R}$.
The following three conditions, defined for a marking
$\Mark(\vpos{a}{i})$, make it easy to compare known acyclicity properties.

\begin{itemize}
	\item {\bf (P1)} $\Gamma^+(\vpos{a}{i}) \subseteq \Mark(\vpos{a}{i})$,
	\footnote{For any node $v$, $\Gamma^+(v)$ denotes the set of (direct) successors of $v$.}
	\item {\bf (P2)} for all $\vpos{a'}{i'} \in \Mark(\vpos{a}{i})$ such that $\vpos{a'}{i'}$ occurs in some
	rule head: 
		$\Gamma^+(\vpos{a'}{i'}) \subseteq \Mark(\vpos{a}{i})$,
	\item {\bf (P3)} for all variable $v$ in a rule body, such that for all
	position $\vpos{a'}{i'}$ with $\fun{term}(\vpos{a'}{i'}) = v$, there is $\vpos{a''}{i'} \in \Mark(\vpos{a}{i})$
	with $\fun{pred}(\vpos{a'}{i'}) = \fun{pred}(\vpos{a''}{i'})$ and $\fun{term}(\vpos{a''}{i'}) = v$:
	$\Gamma^+(v) \subseteq \Mark(\vpos{a}{i})$, where $\Gamma^+(v)$ is the union of all $\Gamma^+(p)$, 
	 where $p$ is an atom position in which $v$ occurs.
\end{itemize}

{\bf (P1)} ensures that the marking of a given node 
includes its successors ;
{\bf (P2)} ensures that the marking
includes the successors of all marked nodes from a rule head
;
and {\bf (P3)} ensures that for each frontier variable of a rule such that all predicate
positions where it occurs are marked, the marking includes its successors.

\begin{definition}[Weak-acyclicity marking]
	A marking $\Mark$ is a $\wa$-marking wrt $X$ if for any $\vpos{a}{i} \in \xpg$,
	$\Mark(\vpos{a}{i})$ is the minimal set such that:
	\begin{itemize}
		\item {\bf (P1)} holds,
		\item for all $\vpos{a'}{i'} \in \Mark(\vpos{a}{i}), \Gamma^+(\vpos{a'}{i'})
		\subseteq \Mark(\vpos{a}{i})$
	\end{itemize}
\end{definition}

\begin{observation}
	The latter condition implies {\bf (P2)} and {\bf (P3)}.
\end{observation}

\begin{proposition}\label{prop:wa}
	A set of rules $\R$ is $\wa$ iff $\FPG{\R}$ satisfies
	the acyclicity property associated with the $\wa$-marking.
\end{proposition}

\begin{proof}
	If $\R$ is not $\wa$, then there is some cycle in the graph of predicate positions
	going through a {\em special edge}.
	Let $p_i$ be the predicate position where this edge ends, and $z$ be the existential variable
	which occurs in $p_i$.
	Let $M$ be the $\wa$-marking of any existential position $\vpos{a}{i}$ with
	$\fun{pred}(\vpos{a}{i}) = p$ and $\fun{term}(\vpos{a}{i}) = z$.

	{\bf (P1)} ensures that the successors of $\vpos{a}{i}$ are marked;  then, 
	the propagation function will perform a classic  breadth-first traversal of the graph.
	By Proposition \ref{prop:pg-fpg}, 
	to each cycle in the graph of predicate positions of $\R$ corresponds
	a set of cycles in $\FPG{\R}$.
	Since $p_i$ belongs to a cycle, \vpos{a}{i} will obviously be marked by the propagation
	function.
	Hence, $\FPG{\R}$ does not satisy the associated acyclicity property of the $\wa$-marking.

	Conversely, if $\R$ is $\wa$, there is no cycle going through a {\em special edge} in the
	graph of predicate positions of $\R$.
	By Proposition \ref{prop:pg-fpg}, to each cycle in $\FPG{\R}$ corresponds a cycle in the graph of predicate positions of $\R$, hence no cycle in $\FPG{\R}$ goes through an existential position.  \end{proof}

\vspace*{0.3cm}	
	We do not recall here the original definitions of $\fd$,  $\ar$,  $\ja$,  $\swa$. The reader is referred to the papers cited in Section \ref{sec:acyclic} or to \cite{grau2013acyclicity}, where these notions are reformulated with a common vocabulary.

	\begin{definition}[Finite domain marking] A marking $\Mark$ is a $\fd$-marking wrt $X$ if for any $\vpos{a}{i} \in \xpg$, $\Mark(\vpos{a}{i})$ is the minimal set such that: \begin{itemize} \item {\bf (P1)} and {\bf
		(P3)} hold, \item for all $\vpos{a'}{i'} \in \Mark(\vpos{a}{i}), \Gamma^+(\vpos{a'}{i'}) \setminus \{\vpos{a}{i}\} \subseteq \Mark(\vpos{a}{i})$. \end{itemize} \end{definition} 
		
\begin{observation} The latter condition implies {\bf (P2)}.  \end{observation} 
		
\begin{proposition}\label{prop:fd} A set of rules $\R$ is $\fd$ iff $\FPG{\R}$ satisfies the acyclicity property associated with the $\fd$-marking.  \end{proposition} 

\begin{proof} Let $\R$ be a set of rules that is $\fd$, then for each existential position p$_i$ there exists a position p$_j$ for each variable of the frontier in the graph of predicate positions such that p$_j$ does not
belong to a cycle. Given $PG^F(\R)$ we can see that Condition {\bf (P3)} ensures that $\R$ is $\fd$.
\end{proof}

\begin{definition}[Argument restricted marking]
	A marking $\Mark$ is an $\ar$-marking wrt $X$ if for any $\vpos{a}{i} \in \xpg$,
	$\Mark(\vpos{a}{i})$ is the minimal set such that:
	\begin{itemize}
		\item {\bf (P1)}, {\bf (P2)} and {\bf (P3)} hold,
		\item for each existential position $\vpos{a'}{i'}$, $\Gamma^+(\vpos{a'}{i'})
		\subseteq \Mark(\vpos{a}{i})$,
	\end{itemize}
\end{definition}

\begin{observation}
	If $\Mark$ is an $ar$-marking, then for all existential positions $\vpos{a}{i},\vpos{a'}{i'} \in \xpg$, 
	   $\Mark(\vpos{a}{i}) = \Mark(\vpos{a'}{i'})$.
\end{observation}

\begin{proposition}\label{prop:ar}
	A set of rules $\R$ is $\ar$ iff $\FPG{\R}$ satisfies the acyclicity property
	associated with the $\ar$-marking.
\end{proposition}

\begin{proof}
Let $\R$ be a set of rules that is $\ar$, then there exists a ranking on terms (i.e., arguments) such that for each
rule the rank of an existential variable needs to be stricly higher than the rank of each frontier variable in
the body and the rank of a frontier variable in the head has to be higher or equal to the rank of this frontier
variable in the body. The marking process is equivalent to the ranking, in fact each time a node is marked, the rank of a term is incremented.  
If we have a cyclic $\ar$-$marking$, it means that there exists at least one term
rank that does not satisfy the property of argument-restricted. We can see the marking process as a method to compute
an argument ranking.
\end{proof}

\begin{definition}[Joint acyclicity marking]
	A marking $\Mark$ is a $\ja$-marking wrt $X$ if for any $\vpos{a}{i} \in \xpg$,
	$\Mark(\vpos{a}{i})$ is the minimal set such that {\bf (P1)}, {\bf (P2)} and {\bf (P3)} hold.
\end{definition}

\begin{proposition}\label{prop:ja}
	A set of rules $\R$ is $\ja$ iff $\FPG{\R}$ satisfies the acyclicity property associated with the $\ja$-marking.
\end{proposition}

\begin{proof}
	The definition of the $\ja$ propagation function is the same as in \cite{kr11}.
	Indeed the ``Move" set of a position is defined in the same way as the marking.
	Furthermore, by Proposition \ref{prop:pg-fpg}, 
	for any predicate position $p_i$ in the graph of joint-acyclicity, there
	is a cycle going through $p_i$ iff for any existential atom position $\vpos{a}{i}$ such that
	$\fun{pred}(\vpos{a}{i}) = p$, we have $\vpos{a}{i} \in M(\vpos{a}{i})$.
	
\end{proof}

\begin{definition}[Super-weak-acyclicity marking]
	A marking $\Mark$ is a $\swa$-marking wrt $X$ if for any $\vpos{a}{i} \in \xpg$,
	$\Mark(\vpos{a}{i})$ is the minimal set such that :
	\begin{itemize}
		\item {\bf (P1)} and {\bf (P3)} hold,
		\item for all $\vpos{a'}{i'} \in \Mark(\vpos{a}{i})$ occuring in a rule head,
		$\{\vpos{a''}{i'} \in \Gamma^+(\vpos{a'}{i'}) :$ $a'$ and $a''$ unify$\}$
		$\subseteq \Mark(\vpos{a}{i})$.
	\end{itemize}
\end{definition}

\begin{proposition}\label{prop:swa}
	A set of rules $\R$ is $\swa$ iff $\FPG{\R}$ satisfies the acyclicity property associated with the $\swa$-marking.
\end{proposition}

\begin{proof}
	In the original paper of \cite{marnette09}, the definition of $\swa$ was slightly different from
	this marking, but it has been shown in \cite{grau2013acyclicity}, that $\swa$ can be equivalently expressed
	by a ``Move" set similar to $\ja$.
	As for the $\ja$-marking, the definition of the $\swa$-marking corresponds to the definition of its  ``Move" set.
\end{proof}

\paragraph{}
\emph{Proof of Proposition \ref{prop:markings}:}
	Follows from Propositions \ref{prop:wa}, \ref{prop:fd}, \ref{prop:ar}, \ref{prop:ja}, and \ref{prop:swa}.